\documentclass[final,12pt]{colt2021} 

\usepackage{bbm}
\usepackage{bbm}
\usepackage{epstopdf}
\usepackage{authblk}
\usepackage{amsfonts,amssymb,mathtools,color,float,graphicx,verbatim}
\usepackage{soul}
\makeatletter
\newcommand{\printfnsymbol}[1]{%
  \textsuperscript{\@fnsymbol{#1}}%
}
\makeatother

\newcommand{\SQdim}{\textrm{SQdim}}
\newcommand{\SQ}{\textrm{SQ}}

\newcommand{\ca}{{\cal A}}

\newcommand{\cd}{{\cal D}}

\newcommand{\ch}{{\cal H}}

\newcommand{\cf}{{\cal F}}

\newcommand{\cx}{{\cal X}}
\newcommand{\cy}{{\cal Y}}
\newcommand{\cz}{{\cal Z}}

\newcommand{\cn}{{\cal N}}

\newcommand{\bh}{{\mathbf h}}
\newcommand{\bb}{{\mathbf b}}
\newcommand{\x}{{\mathbf x}}

\newcommand{\y}{{\mathbf y}}

\newcommand{\z}{{\mathbf z}}

\newcommand{\bu}{{\mathbf u}}
\newcommand{\bw}{{\mathbf w}}

\newcommand{\bv}{{\mathbf v}}

\DeclareMathOperator*{\sign}{sign}

\newcommand{\reals}{{\mathbb R}}
\newcommand{\integers}{{\mathbb Z}}
\newcommand{\naturals}{{\mathbb N}}

\newcommand{\ind}{\mathbf{1}}
\newcommand{\Zero}{\mathbf{0}}

\newcommand{\prob}[1]{{\mathbb{P}}\left[ #1 \right] }
\newcommand{\mean}[1]{{\mathbb{E}}\left[ #1 \right] }

\newcommand{\abs}[1]{\left \lvert #1 \right \rvert}

\DeclareMathOperator*{\E}{\mathbb{E}}

\newcommand{\inner}[1]{\left\langle #1 \right\rangle}

\usepackage{enumitem,kantlipsum}

\newcommand{\bx}{\mathbf{x}}
\newcommand{\bz}{\mathbf{z}}

\newcommand{\Fcal}{\mathcal{F}}

\newcommand{\norm}[1]{\|#1\|}
\newcommand{\secref}[1]{Sec.~\ref{#1}}

\renewcommand{\eqref}[1]{Eq.~(\ref{#1})}
\newcommand{\lemref}[1]{Lemma~\ref{#1}}
\newcommand{\thmref}[1]{Thm.~\ref{#1}}

\newcommand{\appref}[1]{Appendix~\ref{#1}}

\newcommand\revision[1]{\textcolor{black}{#1}}

\title[]{The Connection Between Approximation, Depth Separation and Learnability in Neural Networks}
\usepackage{times}

\coltauthor{%
 \Name{Eran Malach\thanks{equal contribution}}
 \vspace{-0.15in}\Email{eran.malach@mail.huji.ac.il}\\
 \addr{School of Computer Science, Hebrew University} \\
 \Name{Gilad Yehudai\printfnsymbol{1}} \Email{gilad.yehudai@weizmann.ac.il}\\
 \addr{Weizmann Institute of Science} \\
 \Name{Shai Shalev-Shwartz}\Email{shais@cs.huji.ac.il}\\
 \addr{School of Computer Science, Hebrew University} \\
 \Name{Ohad Shamir}
 \Email{ohad.shamir@weizmann.ac.il}\\
 \addr{Weizmann Institute of Science} \\
}

\begin{document}

\maketitle

\begin{abstract}
Several recent works have shown separation results between deep neural networks, and hypothesis classes with inferior approximation capacity such as shallow networks or kernel classes. On the other hand, the fact that deep networks can efficiently express a target function does not mean that this target function can be learned efficiently by deep neural networks. In this work we study the intricate connection between learnability and approximation capacity. We show that learnability with deep networks of a target function depends on the ability of simpler classes to approximate the target. Specifically, we show that a necessary condition for a function to be learnable by gradient descent on deep neural networks is to be able to approximate the function, at least in a weak sense, with shallow neural networks. We also show that a class of functions can be learned by an efficient statistical query algorithm if and only if it can be approximated in a weak sense by some kernel class. We give several examples of functions which demonstrate depth separation, and conclude that they cannot be efficiently learned, even by a hypothesis class that can efficiently approximate them. 
\end{abstract}

\section{Introduction}

The empirical success of deep networks has inspired a large number of theoretical works trying to understand what properties of deep neural networks make them so powerful.
From a theoretical perspective, the success of deep networks is often attributed to their \textit{approximation capacity}.
Deep networks can efficiently implement arbitrary Boolean circuits, and thus can efficiently compute anything that can be efficiently computed by Turing machines. Therefore, in terms of expressive power, deep neural networks are the ultimate choice of hypothesis class. 

In contrast, other hypothesis classes studied in the literature have inferior approximation capacity. For example, Kernel methods (i.e., linear functions in RKHS space) can approximate arbitrary functions, at the cost of having an exponential dimension or margin complexity \citep{rahimi2008uniform, sun2018approximation}. Shallow (two-layer) neural networks can also approximate almost any target function (e.g. \cite{cybenko1989approximation, leshno1993multilayer}), although possibly using an exponentially large number of neurons. 

Indeed, several theoretical studies have demonstrated \textit{separation results}: explicit constructions of functions or function families that can be expressed using deep networks, but cannot be approximated using shallow networks or kernel predictors of reasonably bounded size. For example, \cite{eldan2016power, safran2017depth, daniely2017depth} showed separation results between depth-2 and depth-3 neural networks while \cite{telgarsky2016benefits} showed separation between depth $n$ and depth $n^{1/3}$ neural networks when the input dimension is constant.
The works of \cite{daniely2020learning, yehudai2019power, kamath2020approximate, allen2019can} show separation results between neural networks and kernel methods. 


With that said, all the above results only show an analysis of the approximation capacity of the hypothesis classes. This is unsatisfactory, since the fact that a certain hypothesis class can express some target class does not mean that there is an efficient algorithm that can learn it. Namely, given access to limited computational resources, we would hope to use hypothesis classes for which we have efficient algorithms that can recover the best hypothesis within the class.


In some sense, the point of depth separation results is to argue that depth is beneficial by showing function classes that cannot be efficiently expressed without depth. In this work we show, perhaps surprisingly, that essentially \emph{there are no} target functions which are both efficiently learnable, and "truly deep" in the sense that shallow networks (or kernel classes) cannot even weakly approximate them. More concretely, we explore the intricate connection between approximation capacity and efficient learnability. We first define the notion of \emph{weak approximation} (Definition \ref{def:weak approximation}), namely, a hypothesis class weakly approximates a target class if for every target function there exists some hypothesis that approximates the target slightly better than the trivial predictor. We then show, in different settings and for different learning algorithms, that there is a dependence between the success of the algorithm in learning a target class, and having weak approximation using a ``simple'' class. That is, we show that problems which are hard to weakly approximate using ``simple'' classes (e.g. shallow neural networks) are also hard to learn using the more ``complex'' class (e.g. deep neural network). The above is true, even for target classes that can be exactly represented by the ``complex'' hypothesis class. 

Our main contributions are as follows:
\begin{enumerate}
    \item \textbf{Gradient descent:} A target class of functions cannot be learned by gradient descent on deep neural networks, if $3$-layer neural network cannot weakly approximate it.
    \item \textbf{SQ algorithms:} A target class of functions can be learned by deep neural networks using any statistical query algorithm if and only if they can be weakly approximated by a kernel class of functions.
    \item We consider two known examples of target classes which separate between hypothesis classes, and as a corollary of the previous results, we get that these target classes cannot be learned, even by hypothesis classes that can perfectly represent them. Namely, (1) Telgarski's triangle function \cite{telgarsky2016benefits} cannot be learned using gradient descent; (2) Parity functions cannot be learned by an efficient SQ-algorithm.
    \item We show a specific target class that separates between 2-layer and 3-layer neural networks, and prove that this target class cannot be efficiently learned by any SQ-algorithm.
\end{enumerate}

These results show that the power of deep neural networks to approximate target functions is not enough. In order for deep networks to learn some target function, simpler models such as shallow networks or kernel classes should be able to approximate the target function, at least in a weak sense. 

In \cite{malach2019deeper} it was conjectured that a target class which cannot be approximated by a shallow neural network cannot be learned using gradient-based methods, even when using deep neural networks. Here we give a partial positive answer to this conjecture, for the specific case of learning with gradient descent, and for the approximation capabilities of 3-layer neural networks.

\subsection{Related Works}
\textbf{Hardness of learning with GD.} Several recent works have shown specific scenarios where gradient-based methods fails. In \cite{shalev2017failures, shamir2018distribution} several of failures of gradient descent are shown, including distributions that are hard to learn, and certain hard target functions. In \cite{yehudai2020learning} it was shown that even for the case of learning the simplest neural network, containing a single neuron, and in the realizable case, there are distributions and activation functions which are hard to learn with gradient methods.
In \cite{malach2019deeper} it was shown that for certain fractal distribution, learnability using gradient descent on deep neural networks depends on whether the target can be approximated using shallow networks. Our work can be seen as a generalization of this work to a much larger class of distributions, and to different learning setups (i.e. SQ-algorithms).

\textbf{Separation Results.} In \cite{telgarsky2016benefits}, a family of target functions was introduced that can be realized by depth-$n$ neural networks with polynomial width, but cannot be realized by depth-$n^{1/3}$ neural networks unless the width is exponential. This result was generalized in \cite{chatziafratis2019depth, chatziafratis2020better} to more families of target functions.
Several works \cite{yehudai2019power, allen2019can, kamath2020approximate} have shown separation between kernel methods and shallow neural networks. In particular, in \cite{yehudai2019power} it is shown that kernel methods (including NTK) cannot efficiently approximate a single ReLU neuron, while this problem can be learned with gradient methods using 2-layer network (e.g. \cite{yehudai2020learning, frei2020agnostic}).

Several works have shown target functions that can be well approximated by 3-layer neural networks, while they cannot be approximated by  2-layer networks unless their width is exponential. Such works include \cite{eldan2016power, safran2017depth, safran2019depth} where isotropic functions are considered, \cite{daniely2017depth} where a composition of inner product with complex function which cannot be approximated  by low degree polynomials, and \cite{malach2019learning} where boolean functions are considered. In \cite{vardi2020neural} it is shown that there are natural proof barriers for proving such depth separation results for depth larger than $4$.

\textbf{SQ results.}
Since its introduction in a seminal work by Kearns \cite{kearns1998efficient}, the statistical-query (SQ) framework has been extensively studied in various works. Unlike standard PAC learning, where the learner has access to a set of sampled examples, in SQ learning the learner can use statistical properties of the data, but not individual examples. These statistical properties are provided via access to an oracle, which given some query on the distribution, returns an approximate evaluation of the query. In his original work, Kearns demonstrated that an SQ algorithm can be easily adapted to a noise-robust learning algorithm \cite{kearns1998efficient}. The work of \cite{blum1994weakly} introduced the SQ-dimension, a statistical measure of the target class, that can be used to characterize weak learning with statistical-queries. Such characterization has been extended to other variants of the SQ framework, including strong learning and distribution-free learning \cite{feldman2012complete, simon2007characterization, szorenyi2009characterizing, feldman2012computational}.
Importantly, the SQ framework has been used to derive lower bounds on complexity of learning various problems, for example learning parities \cite{kearns1998efficient} and neural networks \cite{goel2020superpolynomial}.


\section{Definitions and Notations}
\subsection{Preliminaries}
We denote by $\cx_n$ the input space and by $\cy$ the label space. We denote vectors in bold. We denote the sign function as $\sign(x) = 1$ if $x \geq 0$ and $-1$ otherwise.
We focus on binary classification tasks, so $\cy = \{\pm 1\}$.
We consider two classes of functions: 
\begin{itemize}
\item The \textbf{target class}, denoted by $\cf_n$, which is a class of functions from $\cx_n$ to $\cy$ that labels the underlying distribution, and which the learner needs to approximate.
\item The \textbf{hypothesis class}, denoted by $\ch_n$, which is a class of functions from $\cx_n$ to $\reals$ from which the learner can choose its hypothesis.
\end{itemize}
We denote by $\ell : \cy \times \reals \to \reals$ our loss function. Since we consider classification tasks, we assume throughout the paper that $\ell$ is the hinge-loss, namely $\ell(y,\hat{y}) = \max \{1-y\hat{y},0\}$.

For some distribution $\cd$ over $\cx_n$ and some target function $f : \cx_n \to \cy$ ($f \in \cf_n$) we denote by $f(\cd)$ the distribution over $\cx_n \times \cy$ where $\x \sim \cd$ and $y = f(\x)$.
For some hypothesis $h : \cx_n \to \reals$ ($h \in \ch_n$) we denote the loss of $h$ on the distribution $f(\cd)$ by:
\[
L_{f(\cd)}(h) = \E_{(\x,y) \sim f(\cd)} \ell(y,h(\x)) = \E_{\x \sim \cd} \ell(f(\x),h(\x))
\]
We say that a function $r:\naturals\rightarrow\naturals$ is super-polynomial if for every polynomial $p$ we have $\lim_{n\rightarrow\infty}p(n)/r(n) = 0$. For $\bx_0\in\reals^d$ and $r > 0$ we define the ball of radius $r$ around $\bx_0$ as $B_r(\bx_0) = \{\bx\in\reals^d: \norm{\bx-\bx_0}\leq r\}$. We denote by $[n]$ the set $\{1,\dots,n\}$ for $n\in\naturals$.

We will use the following definition of neural networks with bounded width:

\begin{definition}
Let $d\in\naturals$ be the data input dimension. We define a neural network of depth $k$ and width at most $p$ as $h:\reals^d\rightarrow\reals$ with $h(\bx) = \bh^{(k)}\circ\cdots\circ \bh^{(1)}(\bx)$ where:
\begin{itemize}
    \item $\bh^{(1)} = \sigma\left(W^{(1)}\bx + \bb^{(1)}\right)$ for $W^{(1)}\in\reals^{p_1\times d}, \bb^{(1)}\in\reals^{p_1}$
    \item $\bh^{(i)} = \sigma\left(W^{(i)}\bh^{(i-1)} + \bb^{(i)}\right)$ for $W^{(i)}\in\reals^{p_{i}\times p_{i-1}}, \bb^{(i)}\in\reals^{p_i}$ for $i=2,\dots,k-1$
    \item $\bh^{(k)} = W^{(k)}\bh^{(k-1)} + \bb^{(k)}$ for $W^{(k)}\in\reals^{1\times p_{k-1}}, \bb^{(k)}\in\reals^{1}$
\end{itemize}
where $\sigma:\mathbb{R}\rightarrow\mathbb{R}$ is some non-linear function, and $p_i \leq p$ for all $i$.
\end{definition}

\subsection{Weak Approximation and Weak Dependence}

Consider the problem of learning a distribution labeled by some target class $\cf$, using a learning algorithm that can output a hypothesis in $\ch$. The goal of the learning algorithm is to find a hypothesis $h \in \ch$ that minimizes the loss $L_{f(\cd)}(h)$, when given access to examples from a distribution $f(\cd)$ for some $f \in \cf$.

We will consider in the paper two problems which are connected. First, we want to understand whether the fact that the algorithm is forced to output a hypothesis in $\ch$ limits its ability to \emph{approximate} $\cf$. What we could hope for is that the hypothesis class can be expressive enough in order to approximate the target up to some small accuracy:


\begin{definition}
Let $\ch = \{\ch_n \}_{n \in \naturals}$ be a sequence of hypothesis classes, let $\cf = \{\cf_n \}_{n \in \naturals}$ be a sequence of target classes and let $\cd = \{\cd_n \}_{n \in \naturals}$ be a sequence of distributions over $\{\cx_n\}_{n \in \naturals}$. We say that $\ch$ $\epsilon$ \textbf{- approximates} $(\cf,\cd)$ if there exists $\epsilon \in [0,1)$ such that for all $n\in\mathbb{N}$:
\[
\sup_{f \in \cf_n} \inf_{h \in \ch_n} L_{f(\cd_n)}(h) \leq \epsilon~.
\]
We say $(\cf,\cd)$ is \textbf{realizable} by $\ch$ if $\ch$ can $0$-approximate it.
\end{definition}

Note that in the above definition we did not require $\epsilon$ to be small, just that it's a constant which does not depend on $n$. In some cases we cannot guarantee realizability, or even $\epsilon$-approximation for a constant $\epsilon$. The minimal requirement in this case is that the hypothesis class can approximate the target just a bit better than the trivial approximation:


\begin{definition}\label{def:weak approximation}
Let $\ch = \{\ch_n \}_{n \in \naturals}$ be a sequence of hypothesis classes, let $\cf = \{\cf_n \}_{n \in \naturals}$ a sequence of target classes and let $\cd = \{\cd_n \}_{n \in \naturals}$ be a sequence of distributions over $\{\cx_n\}_{n \in \naturals}$. We say that $\ch$ \textbf{weakly approximates} $\cf$ with respect to $\cd$, if there exists some polynomial $p$ such that for all $n\in\mathbb{N}$:
\[
\sup_{f \in \cf_n} \inf_{h \in \ch_n} L_{f(\cd_n)}(h) \leq 1 - 1/\abs{p(n)}
\]
\end{definition}

Since we use the hinge loss, the loss on the zero hypothesis (i.e. output $0$ for every input) is exactly 1. This means that the weak approximation requirement in this case is that the hypothesis class can approximate the target better than the trivial classifier, at least up to an inverse polynomial.

The main goal of this paper is to explore the relation between learnability and approximation. To do so, we define the notion of \textit{weak dependence} between some algorithm $\ca$ and a hypothesis class $\ch$:

\begin{definition}
Let $\ca$ be a learning algorithm and let $\ch = \{\ch_n\}_{n \in \naturals}$ be a sequence of hypothesis classes. We say that $\ca$ \textbf{weakly depends} on $\ch$, if every class-distribution pair $(\cf, \cd)$ that cannot be weakly approximated by $\ch$, cannot be efficiently learned by $\ca$.
\end{definition}

We leave the exact definition of efficient learnability for the next sections, as it depends on the specific setting of learning that we consider. Note that the algorithm $\ca$ does not necessarily output a hypothesis from $\ch$. In general, $\ca$ will output a hypothesis from a "more expressive" class than $\ch$, e.g. where $\ch$ consists of $3$-layer neural networks while $\ca$ outputs a deep neural network.

Clearly, if some learning algorithm $\ca$ outputs a hypothesis from some class $\ch$, then in order for $\ca$ to succeed in learning, it must hold that $\ch$ can (weakly) approximate the target class $\cf$. So, any algorithm $\ca$ which outputs a hypothesis from a class $\ch$ \textit{weakly depends} on the hypothesis class of its output. However, it turns out that some algorithms weakly depend on hypothesis classes that are very different, and sometimes much ``weaker'', than the class that is being learned by the algorithm. We next show some notable examples of such dependencies.

\revision{We note that in all the above definitions we considered a sequence of target classes $\Fcal$ parameterized by some parameter $n$. In the literature, there are several kinds of depth separation results, where some parameter $n$ tends to infinity, and this parameter corresponds to some property of the problem. For example, in \cite{eldan2016power,safran2017depth} $n$ is the input dimension, while in \cite{telgarsky2016benefits} $n$ is the depth of the network. A major benefit of our terminology is that we give a single definition for both kinds of depth separation.}



\section{Gradient Descent Weakly Depends on Shallow Neural-Networks}\label{sec:GD depends}
In this section we focus on the gradient descent algorithm. We show that under certain technical assumptions, the gradient descent algorithm applied to  \emph{deep} neural networks weakly depends on shallow $3$-layer neural networks.

First, let us define the algorithm that is being used. Let $g_\theta(\bx):\mathbb{R}^d\rightarrow\mathbb{R}$ be some function parameterized by a vector $\theta\in\reals^r$, let $f:\mathbb{R}^d\rightarrow\reals$ be a target function and $\cd$ a distribution over $\reals^d$. Suppose we initialize the parameter vector at $\theta_0\in\reals^r$, for a learning rate $\eta > 0$ the \textbf{gradient descent} algorithm iteratively computes $\theta_t$ by the following rule:
\[
\theta_t = \theta_{t-1} - \eta\nabla_{\theta_{t-1}} L_{f(\cd)}(g_{\theta_{t-1}}).
\]
For example, we can think of $g_\theta(\bx)$ as a neural network with $r$ parameters, taking as input $d$-dimensional data. We define when a class of functions is \emph{not} weakly learnable by gradient descent in the following way:

\begin{definition}
Let $\cf=\{f_n:\reals^d\rightarrow\reals\}_{n\in\naturals}$ a sequence of target classes, $\cd$ a distributions over $\cx\subseteq \reals^d$, $g=\{g_{\theta}^n(\bx):\cx\rightarrow\reals\}_{n\in\naturals}$ a sequence of functions parameterized by a vector $\theta\in\reals^{q(n)}$ where $q(n)$ is some polynomial, and $\theta_0=\{\theta_0^n\in\reals^{q(n)}\}_{n\in\naturals}$ be a sequence of initialization points. Let $T(n),\eta(n):\naturals\rightarrow\reals$ be two polynomials. We say that $\cf$ is \textbf{not weakly learnable by gradient descent} with respect to $\cd$ using the functions g at initialization $\theta_0$, if there exists a super polynomial function $\alpha:\naturals\rightarrow\naturals$ such that for every $n\in\naturals$, running gradient descent on $g^n_{\theta_0^n}$ for $T(n)$ iterations, and any learning rate $\eta \leq \eta(n)$ we have that 
\revision{
\[
 L_{f_n(\cd)}(g_{\theta_0}^n) - L_{f_n(\cd)}(g_{\theta_T}^n) \leq \frac{1}{\alpha(n)}~.
\]
}
Here, for ease of notation we denote by $g^n_{\theta_0}$ the function $g^n_\theta$ initialized at $\theta_0^n$, and by $g^n_{\theta_T}$ the function $g^n_\theta$ after $T$ iterations of gradient descent, initialized at $\theta_0^n$
\end{definition}



The definition contains many parameters, but it is actually quite intuitive. In simple words, gradient descent is unable to learn a target class (and distribution) if after a polynomial number of iterations, the loss stays very close to the loss at the initialization. We give this definition in negation ("not weakly learnable") because getting the loss away from its initialized value is a necessary condition to learn (assuming the loss is not so good at initialization), but it is not sufficient. Note that our only requirement of the learned function $g$ is that the number of optimized parameters is polynomial, as optimizing a super-polynomial number of parameters is practically intractable. Although we give this definition with a constant learning rate, all the results can be readily extended to GD with variable learning rates, as long as they are all smaller than $1$.

Next, we focus on the initialization scheme that is being used, we consider the following initialization of the parameter $\theta$:

\begin{definition}\label{def:L-standard}
Let $g_\theta(\bx):\reals^d\rightarrow\reals$ be some function parameterized by $\theta\in\reals^r$. We say that $\theta_0\in\reals^r$ is an \textbf{$L$-standard initialization} if there is a $\rho >0$ such that for every $\theta\in B_{\rho}(\theta_0)$:
\begin{itemize}
    \item $g_\theta(\bx)$ is an $L$-Lipschitz function of $\theta$
    \item Each coordinate of $\nabla_\theta(g_\theta(\bx))$ is an $L$-Lipschitz function of $\bx$ with $\sup_{\bx\in[0,1]^d}\nabla_{\theta}\left(g_{\theta}^n\right)_i(\bx)\leq L$ for every $i\in[r]$.
\end{itemize}
\end{definition}

In \appref{appen:standard init} we show that Xavier initialization for depth-$k$ neural networks is w.h.p an $L$-standard initialization, where $L=O(d)$ ($d$ being the input dimension) and $\rho =\frac{1}{k}$. 


Using the above definitions, we can show that gradient descent on a class of functions $g$ weakly depends on the hypothesis class of $3$-layer neural networks with bounded width.


\revision{\begin{theorem}\label{thm:weak learnability vary depth}
Let $\alpha:\mathbb{N}\rightarrow\mathbb{N}$ be a super-polynomial function, $d\in\naturals$ the input dimension for the data, and $\cd$ the uniform distribution on $[0,1]^d$. Let $\mathcal{F} = \{f_n:[0,1]^d\rightarrow\mathbb{R}\}_{n\in\mathbb{N}}$ be a sequence of functions, with $|f_n(\bx)| \leq C$ for all $n$ for some constant $C >0$. Assume that the sequence of hypothesis classes of  $3$-layer neural networks with ReLU activations and width at most $\alpha(n)^d\cdot 2d$ cannot weakly approximate  $\mathcal{F}$ in the sense that for any $n\in\naturals$,  $\min_{h\in\mathcal{H}_n}L_{f_n}(h) \geq 1 - \alpha(n)^{-1}$.
Let $g = \{g^n_\theta(\bx)\}_{n\in\naturals}$ be a function sequence parameterized by a vector $\theta\in\reals^{p(n)}$ with a polynomial $p$, and assume we initialize at $\theta_0\in\reals^{p(n)}$ which is an $L$-standard initialization $\theta_0$ with $\rho \geq \frac{1}{n}$. Then, running gradient descent for $T+1$ iterations with learning rate $\eta$ we have that:
\[
L_{f_n}(g_{\theta_0}^n) - L_{f_n}(g_{\theta_{T+1}}^n) \leq  \frac{21L^2C^2\max\{1,\eta^2\}\sqrt{d}p(n)}{\alpha(n)}T^2~,
\]
\end{theorem}}

The full proof can be found in \appref{appen:Weak learnability vary depth}. The proof intuition is as follows: We show that the correlation between each coordinate of $\nabla_\theta g_\theta^n$ and $f_n$ cannot be too large. We do that by using an approximation of Lipschitz functions with $3$-layer neural networks. Using an argument regarding the optimization process of gradient descent and the Lipschitz initialization assumption on $g_\theta^n$, we show that even after a polynomial number of iterations, the correlation between $g_\theta^n$ and $f_n$ must remain small. This shows that the loss after a polynomial number of iterations cannot be too far away from the loss at the initialization. 

\revision{
\begin{remark} Some remarks on applying a similar analysis for other variants of gradient-descent:
\begin{enumerate}
    \item Our result holds for ``exact'' gradient-descent, where the update of the weights is done using the exact value of the population gradient. We note that, with some further assumptions on the initialization procedure, this result can be extended to ``noisy'' gradient-descent, where an i.i.d. noise is added to the weights at each iteration of gradient-descent. Indeed, we show that in the area of the initialization, the gradients are extremely small. Adding noise at each iteration is equivalent to re-initializing the model in a new (random) initialization, and then applying the extremely small gradients observed so far. So, if such random initialization is $L$-standard w.h.p., then a similar result can be derived for ``noisy'' gradient-descent.
    \item We note that showing a similar result for SGD, where the gradients are calculated based on examples sampled from the distribution, is far more tricky. In fact, in \cite{abbe2020poly} it is shown that neural networks trained using SGD with a batch size of $1$ can implement any poly-time PAC algorithm. Hence, showing hardness results in this setting is as complicated as showing hardness results on PAC learning, which involves relying on unproved comptuational hardness assumptions (e.g., cryptographic hardness).
\end{enumerate}
\end{remark}}

\revision{We get the following as an immediate corollary from \thmref{thm:weak learnability vary depth} by assuming that both $\eta$ and $T$ are at most polynomial in $n$:
\begin{corollary}
Under the same assumptions as in \thmref{thm:weak learnability vary depth}, the function sequence $\Fcal$ is not weakly learnable by $g$ with gradient descent that is initialized at any $L$-standard initialization with $\rho \geq \frac{1}{n}$, and run for a polynomial number of iterations, with a learning rate at most polynomial in $n$.
\end{corollary}}


\begin{remark}
A couple of remarks about the assumptions of \thmref{thm:weak learnability vary depth}:
\begin{enumerate}
    \item The notion of weak learnability in the theorem is the same as in Definition \ref{def:weak approximation}. Here we omitted the $\max_{f \in \cf_n}$ since for any $n$, $\cf_n$ contains only a single function. Also, in the theorem we explicitly used the assumption that there is a super polynomial function which bounds the approximation.
    \item \revision{Definition \ref{def:L-standard} can be satisfied by neural networks with differentiable activations, but not with the ReLU activation. To apply the theorem specifically to ReLU activation would require revising \lemref{lem:approx lipschitz with NN} to the derivative of a neural network with ReLU activation. We believe this can be done if we assume a uniform distribution on $[0,1]^d$, and leave it for future work.} 
    \item The assumption that $\rho \geq \frac{1}{n}$ is weaker than assuming that $\rho$ is a constant, since here we allowed the radius for which the initialization is well behaved to get smaller with $n$.
\end{enumerate}
\end{remark}

One caveat of the theorem is the exponential dependence of network width in the dimension. This dependency is due to the approximation of high dimensional Lipschitz functions using shallow networks. It can be seen that a target class of functions $\cf$ which satisfies the requirement of the theorem, has a Lipschitz constant which is exponential in $d$. The unlearnability itself comes from the fact that the target function has a large Lipschitz constant, although in these results we treat the input dimension $d$ as a constant, while the depth $n$ varies. In \cite{vardi2020size} it is shown that it may not be possible to find a function with a similar depth separation property that has a polynomial Lipschitz constant.

With that said, we immediately get from \thmref{thm:weak learnability vary depth} and \thmref{thm: xavier standard init} the following:

\begin{corollary}
Let $\ca$ be the following learning algorithm: For any $n$, initialize a neural network with depth $n$ and width $p(n)$ (for $p(n)$ polynomial) using standard Xavier initialization. Then, train the neural network with gradient descent with step size $\eta \leq 1$. Then $\ca$ weakly depends on the hypothesis class of $3$-layer neural networks with super-polynomial width.
\end{corollary}

In \secref{sec:strong separation} we will use this corollary to give an example of a function that, although it can be realized by depth $n$
neural networks, it cannot be learned by them.



\section{Statistical-Query Algorithms Weakly Depend on Kernel Classes}
\label{subsec:learnability linear classes}

In this section we relate learnability using Statistical-Query (SQ) algorithms and approximation using a kernel class. Specifically, we show that efficient learning in the SQ model weakly depends on kernel classes. We start by defining weak learnability in the SQ model, following the definitions of \cite{kearns1998efficient}. First, for some function $f$ distribution $\cd$ and tolerance parameter $\tau > 0$, we define the statistical-query oracle $\SQ_\tau(f,\cd)$ to be an oracle which accepts queries of the form $q : \cx \times \cy \to [-1,1]$ and returns some value $v$ such that $\abs{v - \E_{\x \sim \cd}{q(\x,f(\x))}} \le \tau$.

\begin{definition}\label{def:weak learnability}
Let $\cf = \{\cf_n \}_{n \in \naturals}$ be a sequence of target classes and let $\cd = \{\cd_n \}_{n \in \naturals}$ be a sequence of distributions over $\{\cx_n\}_{n \in \naturals}$. We say that $\cf$ is \textbf{weakly learnable} with respect to $\cd$, if there exists a sequence of algorithms $\ca = \{\ca_n \}_{n \in \naturals}$ and polynomials $p,q,r$ such that for every $f \in \cf_n$, the algorithm $\ca_n$ returns a hypothesis $h$ such that:
\[
 L_{f(\cd_n)}(h) \leq 1 -  1/\abs{p(n)}
\]
using at most $q(n)$ queries to $\SQ_{1/r(n)}(f,\cd_n)$.
\end{definition}

Our main result in this section shows the weak dependence between SQ algorithms and the class of functions over a polynomial-size kernel space. We define a polynomial-size kernel class as follows:

\begin{definition}
A sequence of hypothesis classes $\ch := \{\ch_n \}_{n \in \naturals}$ is a \textbf{polynomial-size kernel class} if there exist polynomials $p,q$, and a sequence of mappings $\Psi_n: \cx_n \to [-1,1]^{p(n)}$ such that:
\[
\ch_n = \{\x \mapsto \inner{\Psi_n(\x), \bw} ~:~\norm{\bw}_2 \le q(n) \}
\]
\end{definition}

Although this class is significantly less expressive than neural networks, here we show that it is possible to weakly learn a target class of functions with any SQ algorithm if and only if it can be weakly approximated by the kernel class:

\begin{theorem} \label{thm:linear_classes}
Let $\cf = \{\cf_n \}_{n \in \naturals}$ a sequence of target classes and let $\cd = \{\cd_n \}_{n \in \naturals}$ be a sequence of distributions over $\{\cx_n\}_{n \in \naturals}$. Then, there exists an efficient statistical-query algorithm $\ca$ that weakly learns $\cf$ if and only if there exists a polynomial-size kernel class $\ch$ that weakly approximates $\cf$ with respect to $\cd$.
\end{theorem}

The proof intuition is to show that a polynomial size kernel class can weakly approximate $\cf$ if and only if the SQ dimension of $\cf$ is polynomial. We also use a modified form of a known result which shows that the number of queries required for any SQ algorithm to learn a class of functions depend polynomially on the SQ dimension of this class of functions. The full proof can be found in \appref{appen:proofs from weak learning}. The following corollary immediately follows from \thmref{thm:linear_classes}:

\begin{corollary}
Every SQ-algorithm $\ca$ \textbf{weakly depends} on some kernel class $\ch$.
\end{corollary}

Note that SQ-algorithms are of course not limited to learning only kernel classes. In fact, almost any learning algorithm that has been studied in the machine learning literature can be implemented in the SQ framework. However, the above corollary states that although SQ-algorithms can potentially learn very complex function classes, they are limited to learning only classes that can be weakly approximated using a kernel class.

\section{Strong Separation Between Hypothesis Classes}\label{sec:strong separation}

So far, we saw examples of \textit{weak dependence} between algorithms and hypothesis classes. Our results suggest that in some cases, learning a complex hypothesis class (for example, deep neural networks) depends on having a weak approximation using a ``simpler'' class. So, from an optimization perspective, there is no gap between weak learnability of the ``complex'' class and weak approximation of the ``simpler'' class - if the ``simple'' class cannot (weakly) approximate, the ``complex'' class cannot be learned. 

However, we can review the same question from an approximation perspective. Namely, we can consider distributions that can be expressed by the ``complex'' class and cannot be weakly approximated by the ``simple'' class. In this section we show that when we disregard optimization and consider only approximation capacity, we get such extreme gap between the ``simple'' and the ``complex'' class. In this case we say that there is a \textit{strong separation} between the two classes. We define this formally as follows:

\begin{definition}
Let $\ch = \{\ch_n \}_{n \in \naturals}, \ch' = \{\ch'_n \}_{n \in \naturals}$ be two sequences of hypothesis classes, let $\cf = \{\cf_n \}_{n \in \naturals}$ a sequence of target classes and let $\cd = \{\cd_n \}_{n \in \naturals}$ be a sequence of distributions over $\{\cx_n\}_{n \in \naturals}$. We say that $(\cf, \cd)$ \textbf{strongly separates} $\ch$ from $\ch'$, if $\ch'$ can $\epsilon$-approximate it, but it cannot be weakly approximated by $\ch$ with respect to $\cd$.
\end{definition}

The construction of a function which strongly separates two hypothesis classes suggests that one class is significantly more expressive than the other. However, the fact that a hypothesis class can express a target function does not imply that a learning algorithm which outputs a hypothesis from this class can \emph{learn} the target function.

\begin{corollary}\label{cor:strong separation unlearnable}
Let $\ch = \{\ch_n \}_{n \in \naturals}, \ch' = \{\ch'_n \}_{n \in \naturals}$ be two sequences of hypothesis classes, let $\cf = \{\cf_n \}_{n \in \naturals}$ a sequence of target classes and let $\cd = \{\cd_n \}_{n \in \naturals}$ be a sequence of distributions over $\{\cx_n\}_{n \in \naturals}$. Assume that $(\cf,\cd)$ strongly separates  $\ch$ from $\ch'$, then:
\begin{itemize}
    \item If $\ch$ is the class of 3-layer neural networks with super-polynomial width, and $\cx_n=[0,1]^d$ for all $n$, then $(\cf,\cd)$ is not learnable by gradient descent, even when it outputs a hypothesis from $\ch'$.
    \item If $\ch$ is some polynomial size kernel class, , and $\cx_n=\{\pm 1\}^n$, then $(\cf,\cd)$ is not learnable by any SQ algorithm.
\end{itemize}
\end{corollary}

The proof follows directly from the assumption that $(\cf,\cd)$ cannot be weakly approximated by $\ch$, and using either \thmref{thm:linear_classes} or \thmref{thm:weak learnability vary depth}. 

\subsection{Examples of strong separation}
In the following subsection we will consider known examples of functions which strongly separates hypothesis classes. As a result from the previous subsection, we will get that although the target functions can be realized by neural networks, they cannot be learned by standard learning algorithms.

\subsubsection*{Telgarski's function}

Here we assume the input space is $\cx_n = [0,1]^d$ for all $n$. We define the following family of functions on $d$ dimensional vectors: \[
f_n(\x) = \begin{cases} 1 & \exists t \in \naturals, ~x_1 \in [\frac{2t}{2^{n}}, \frac{2t+1}{2^n}] \\ -1 & otherwise \end{cases}~
\]
Let the input distribution $\cd$ to be the uniform distribution over $\cx_n$, and let $\cf_n := \{f_n\}$. This is the sign of Telgarsky's triangle function in $d$ dimensions (see \cite{telgarsky2016benefits})\footnote{To be more precise, let $g(x)$ be Telgarski's triangle function with $2^n$ jumps, defined using composition of ReLUs (Lemma 3.10(1) from \cite{telgarsky2016benefits}). Then $f_n(x) = \text{sign}(g(x) - 0.5)$.}.

Telgarsky shows that the function above cannot be approximated using a neural network with less than $n^{1/3}$ layers up to a constant accuracy. This result can be extended using the same methods to show that this function cannot be approximated using neural networks with less than $n^{1/2}$ layers, even up to \emph{polynomially non-trivial} accuracy. This means that shallow neural network cannot approximate this function significantly better than the trivial predictor. Hence we have the following strong separation between deep and shallow neural networks with constant input dimension:

\begin{theorem}\label{thm:deep network seperation}
Let $k_1,k_2: \naturals \to \naturals$ some functions such that $k_1(n) = n$ and $k_2(n) \le \sqrt{n}$. Then $(\cf,\cd)$ strongly separates the sign of polynomial-size depth-$k_2$ dimension-$d$ networks from the sign of polynomial-size depth-$k_1$ networks. \revision{In particular, for any neural network $N$ with depth $k_2(n)$ and width $p(n)$ we have that:
\[
 L_{f_n(\cd_n)}(N) \geq  1 - \frac{2^{\sqrt{n}}(2p(n))^{\sqrt{n}}}{2^n}~.
\]
}
\end{theorem}

This result is similar to the results from \cite{telgarsky2016benefits}. For completeness, we give here the full proof in our terminology of strong separation. The full proofs are in \appref{appen: proof of deep  separation}. 

Note that in the above separation results the dimension is constant, and the separation is between varying depth neural networks. Also, we use the sign function of the network for ease of the proofs. It is possible to drop this assumption and show hardness of approximating $f_n$ using the same proof techniques as in \cite{telgarsky2016benefits}.

Combining \thmref{thm:deep network seperation} and \corollaryref{cor:strong separation unlearnable} we have shown that Telgarski's function is not weakly learnable by gradient descent with standard initialization. This is true even if the algorithm is performed over deep neural networks which can represent Telgarski's function.

\revision{We note that in our proofs we formally show that the gradient of the objective is exponentially small. One might consider normalized gradient descent, where the norm of the gradient at each iteration is normalized to some fixed value, hence avoiding the problem of small gradients. This solution is impractical, since for finite precision machines which are used in practice, the gradient is so small that it virtually equal to zero.}

\subsubsection*{Parity functions}
Assume the input space is $\cx_n = \{\pm 1\}^n$. The separation here will be with respect to the input dimension, as opposed to the previous example where the input dimension was fixed.

We will use \textbf{parity functions} over $n$-bits: For some subset $I \subseteq [n]$, denote by $f_I(\x) = \prod_{i \in I} x_i$, the parity over the bits of $I$, let $\cf_n = \{f_I ~:~ I \subseteq [n]\}$ and $\cd_n$ the uniform distribution on $\cx_n$. The following result was shown in several previous works (e.g. see \cite{daniely2020learning}):

\begin{theorem}\label{thm:parity strong separation}
Let $\ch$ be the class of depth-two networks with ReLU activation with polynomial width and weight magnitude, and let $\cf=\{\cf_n\}_{n\in\naturals}$ be the class of parity functions and $\{\cd_n\}_{n\in\naturals}$ the uniform distribution of $\cx_n$. Then, $(\cf,\cd)$ strongly separates any polynomial-size kernel class from $\ch$.
\end{theorem}

For completeness, we give the full proof in \appref{appen:parity functions}. Combining \thmref{thm:parity strong separation} and \corollaryref{cor:strong separation unlearnable} shows that parity functions are not learnable by any efficient SQ algorithm.

\section{Depth-2 and Depth-3 Neural Networks}\label{sec:depth 2-3}

Many previous works showed separation between depth-2 and depth-3 neural networks (e.g. \cite{eldan2016power, safran2017depth, daniely2017depth, safran2019depth}). These results usually construct a function that can be approximated up to very small accuracy with a 3-layer network and polynomial width, while it cannot be approximated by a 2-layer network, unless its width is exponential.

We believe that a necessary condition for a target class to be learned by an SQ-algorithm is to be weakly approximated by the class of 2-layer neural networks. 
\revision{That is, we conjecture that any rich enough \footnote{Note that in general, any single function $f$ can be trivially ``learned'' by the SQ algorithm that always returns the function $f$. Therefore, for some class to be hard to learn using SQ algorithms, it needs to be ``rich enough'', i.e. contain a large and diverse set of functions. } class of functions that cannot be approximated by shallow neural networks of polynomial width, cannot be learned by any SQ algorithm.}
If this conjecture is true, then any \revision{``rich''} target class that strongly separates  2-layer from 3-layer neural networks, cannot be learned by an SQ-algorithm. This is true, even if the SQ-algorithm outputs a 3-layer (or deeper) neural network.

In what follows we will show a specific family of functions that strongly separates 2-layer from 3-layer neural networks, and is not weakly learnable by any SQ-algorithm.

\subsection{An Example of Unlearnability and Separation For 2-layer and 3-layer Networks}
Let $\cx_n=\{\pm 1\}^n$, in this section the input space will be $\cx_n\times \cx_n$, and the output space is $\cy = \{\pm 1\}$. Fix some $\bz'\in\cx_n$, we define the following function\footnote{If the input space is $\{0,1\}^n$ this function would be equivalent to $\inner{\bx,\bz}~\text{mod } 2$, where the inner products is taken over the coordinates for which $z_i'=1$. We keep the input space as $\{\pm 1\}^n$ to be consistent with the previous sections.}: 
\[
F_{\z'}(\x,\z) := \prod_{i=1}^n (x_i \vee z_i \vee z_i') = \prod_{i \in I(z')} (x_i \vee z_i)~,
\]
\revision{where $I(z')\subseteq[n]$ is the set of coordinates for which $z'_i=1$.} Define the following family of functions: $\mathcal{F} := \{F_{\z'} ~:~ \z' \in \mathcal{X}\} \subset \mathcal{Y}^{\mathcal{X} \times \mathcal{X}}$.

\begin{theorem}\label{thm:learnability depth2-3}
Let $\cd_n$ be a uniform distribution over $\cx_n\times \cx_n$. Then the family of function $\cf$ defined above is not weakly learnable by any SQ-algorithm. 
\end{theorem}

The proof intuition is to show that class $\cf$ has large $\SQdim$, and hence cannot be learned using SQ-algorithms, regardless of the chosen hypothesis class. We show this by defining the following inner-product: $\inner{F_{\z'}, F_{\z''}} = \E_{\x,\z} \left[ F_{\z'}(\x,\z) F_{\z''}(\x,\z) \right]$, and explicitly finding the vectors in $\cx_n$ that realize the SQ-dimension. The full proof can be found in \appref{append:learnability depth-2-3}.

Suppose we fix some $\bz'$ such that its number of coordinates that equal to $+1$ is $\Omega(n)$. In this case, we can show that the function $F_{\bz'}(\bx,\bz)$ cannot be weakly approximated by 2-layer neural networks with polynomial width. On the other hand, for any $\bz'$, the function $F_\bz'(\bx,\bz)$ can be realized by a 3-layer network with polynomial width and ReLU activation (see e.g. \cite{malach2019learning}). This is summed up in the following:

\begin{theorem}\label{thm:parity-like separates}
Let $\ch$ be a polynomial-size depth-three network class with ReLU activation. Let $F_n:\cx_n\times \cx_n$ with $F_n(\x,\z) = \prod_{i =1}^n (x_i \vee z_i)$ and let $\cd_n$ be the uniform distribution over $\cx_n\times \cx_n$. Then $(\{F_n\},\{\cd_n\})$ strongly separates any polynomial-size depth-two network class from $\ch$.
\end{theorem}

The full proof can be found in \appref{appen:proofs of 2-3-layer}. Although, for simplicity, we proved the theorem for $\bz'$ which equal to $+1$ in all the coordinates, the proof is the same for any $\bz'$ with $\Omega(n)$ coordinates which equal to $+1$.

To conclude, we have shown an example of a family of functions which strongly separates depth-2 from depth-3 neural networks, which cannot be learned by any SQ-algorithm, even if it outputs a 3-layer network which can represent the function.

\section{Discussion and Future Work}
In this work we have shown that there is an intricate connection between the approximation capabilities of "simple" hypotheses classes, and learning capabilities of certain algorithms, even when used on a "complex" hypothesis class that can perfectly represent the target. We have shown this connection appears in two learning setups: The first is that being able to successfully weakly approximate using a 3-layer neural is a necessary condition to successful learning with gradient descent. The second is that to be able to weakly learn a target class (even in a weak sense) by any SQ-algorithm is possible if and only if this target class can be weakly approximated by some kernel class.

We have also discussed two known examples of target classes which strongly separates hypothesis classes. As a consequence of the dependence between approximation and learnability, these target classes (namely, Telgarski's triangle function and parity functions) cannot be efficiently learned, even when a "complex" architecture which can represent the functions is used. Finally, we have shown a specific example for a target class which separates between 2-layer and 3-layer networks, that cannot be learned by any SQ-algorithm.

We note that our work does not cover the separation results from \cite{eldan2016power, safran2017depth, daniely2017depth}. This is because the target functions that are introduced seem to be weakly approximable by 2-layer networks (e.g. see Figure 1 in \cite{safran2017depth}). It is an open question whether these function can be learned up to arbitrary accuracy using 3-layer (or deeper) neural networks. 

Another interesting future direction is find more learning setups for which there is a dependence between approximation with a "simple" hypothesis class, and learning with a more "complex" hypothesis class. Such setups include other gradient methods, e.g. stochastic gradient descent, or momentum, and separation between shallow and deep neural networks where the input dimension is not constant.

\subsubsection*{acknowledgments}
This research is supported by the European Research Council (TheoryDL project), and by European Research Council (ERC) grant 754705.

\bibliography{bib}

\appendix

\section{Discussion on $L$-standard initializations}\label{appen:standard init}
The assumption of an $L$-standard initialization is required, although sometimes implicitly, in order to learn with gradient methods. The intuition for the first part of the definition is that moving the parameters of the functions by a small amount (e.g. doing a single gradient step) doesn't change the value of the function too drastically. Without this requirement, it will be hard to predict the outcome of even a single gradient step with a small step size. The second part of the definition is that the gradient is well behaved with respect to the data. This means that the gradient w.r.t two close data points is similar, without this requirement small perturbations of the data could also change the gradient direction too significantly. The boundness requirement is technical and can be obtained by bounding the data domain.

All of the requirements of $L$-standard initialization are satisfied by using standard Xavier initialization \cite{glorot2010understanding}. We get a bound on the Lipschitz constant that depends on the norm of $\bx$ (which in our case can be bounded by $d$), and in a radius that depends on the depth of the network.

\begin{theorem}\label{thm: xavier standard init}
Let $\sigma$ be a $1$-Lipschitz activation function and let $d>0$ be the data dimension. Suppose we initialize a neural network of depth $k$ and width $m$ using Xavier initialization, and assume that $m > k^2$. Then w.p $> 1-e^{-\Omega(m/k)}$, this is an $1.1d$-standard initialization with radius $\rho = \frac{1}{k}$
\end{theorem}

In Xavier initialization, each weight entry in a weight matrix of width $m$ is drawn from \revision{$\cn\left(0,\frac{1}{{m}}\right)$} (up to a constant factor). For a depth $k$ width $m$ network it will be easier to view this initialization as if each weight is drawn from a standard Gaussian distribution, and the network is multiplied by the normalization term $\frac{1}{m^{k/2}}$ \revision{(which is also done in \cite{du2019width})}.

For simplicity of the proof we assume the network doesn't have bias terms, and that the width of all the weight matrices is the same. The main technical part of the proof is using Lemma 6.1 from \cite{du2019width} which uses tail bounds on certain random variables to bound the norm of the multiplication of random matrices where each entry is drawn from a standard Gaussian initialization. 

\begin{proof}
We can write the network as:
\[
N(\bx) = W_k\sigma(W_{k-1}\cdots\sigma(W_1\bx)\cdots)~,
\]
where $\bx \in \reals^d$, $W_1\in \reals^{m\times d}$, $W_i\in \reals^{m\times m}$ for $i=2,\dots,k-1$ and $W_k\in \reals^{m\times 1}$. For the first part of the definition, recall that the Lipschitz constant of a composition of functions can be bounded by the multiplication of the Lipschitz constants of the functions. Since $\sigma$ is $1$-Lipschitz, and a neural network is a composition of matrix multiplication with the activation, it is enough to show the Lipschitzness for the linear network function:
\[
N'(\bx) = W_k\cdots W_1\bx~.
\]
Let $W_1,\dots,W_k$ be the weight matrices at initialization, and $A_1,\dots,A_k$ be some perturbation matrices with $\max_{i}\norm{A_i}\leq k$. Suppose for this part of the proof that $\norm{\bx} = 1$ (we will deal with $\bx$ with other norms later). Denote by $W_{(-i_1,\dots,-i_{\ell})}$ for $i_1,\dots,i_\ell\in \{1,\dots,k\}$ the multiplication of the matrices $W_k,\dots, W_1$ without the $i_1,\dots,i_\ell$ matrices. Then we have:
\begin{align}\label{eq:lipschitz first part}
    &\norm{(W_k + A_l)\cdots(W_1+A_1)\bx - W_k\cdots W_1\bx} \\
    &\leq \sum_{i=1}^k\norm{A_iW_{(-i)}\bx} + \sum_{i\neq j}\norm{A_i A_jW_{(-i,-j)}\bx} +...\nonumber
\end{align}
where at each sum there are $\ell$ indices for $\ell =1,\dots,k$ for matrices which are left out of the multiplication. For a single sum we have:
\begin{align*}
     &\sum_{i_1\neq\dots\neq i_\ell}\norm{A_{i_1}\cdots A_{i_\ell}\cdot W_{(-i_1,\dots-i_\ell)}\bx} \leq \max_i\norm{A_i}^\ell \cdot \sum_{i_1\neq\dots\neq i_\ell}\norm{W_{(-i_1,\dots-i_\ell)}\bx}
\end{align*}
We use Lemma 6.1 from \cite{du2019width} to get that w.p $ > 1-e^{-\Omega(m/k)}$ we have that $\norm{W_{(-i_1,\dots-i_\ell)}\bx} \leq 1.1 m^{\frac{k-\ell}{2}}$\footnote{Although the matrix $W_k$ is of a different dimension, the proof can be easily extended to deal with this case.}. Hence, we can bound the above term by:
\begin{align*}
1.1\cdot \max_i\norm{A_i}^\ell \cdot m^{\frac{k-\ell}{2}} \binom{k}{\ell} &\leq 1.1\cdot \frac{1}{k} \cdot m^{\frac{k-\ell}{2}} \cdot k^\ell \\
&\leq \frac{1.1}{k}m^{k/2} \left(\frac{k}{\sqrt{m}}\right)^\ell \leq \frac{1.1m^{k/2}}{k}
\end{align*}
where we used the assumption that $m > k^2$
Since there are $k$ such sums we have that the above can be bounded by $1.1m^{k/2}$, dividing by the normalization term, we get that a Lipschitz constant of $1.1$ with probability $> 1-e^{-\Omega(m/k)}$. If $\norm{x}\neq 1$ then we can divide \eqref{eq:lipschitz first part} by $\norm{\bx}$ to get a Lipschitz and follow the proof in the same manner to get a Lipschitz constant of $1.1\norm{\bx}$. This finishes the first part of the definition.

For the second part, the Lipschitzness condition follows from the fact that 
\[
\frac{\partial N'}{\partial W_i} = (W_k\cdots W_{i+1})^\top (W_{i-1}\cdots W_1\bx)^\top,
\]
from Lemma $6.1$ from \cite{du2019width} and the same reasoning as in the proof of the first part. In this manner we get a Lipschitz constant of $1.1$ w.h.p. For boundness condition, using the assumption that $\bx\in[0,1]^d$ we get that $\max_{\bx\in[0,1]^d}\norm{\bx} = d$, hence we have that the supremum over $\bx$ of the gradient will be bounded by $1.1d$ for every coordinate of the gradient.
\end{proof}

\section{Proof of \thmref{thm:weak learnability vary depth}}\label{appen:Weak learnability vary depth}

The following two lemmas will be necessary in order to approximate a Lipschitz function using a shallow neural network.

\begin{lemma}\label{lem:approx cube with NN}
Let $\gamma >0$ and $A:=[a_1,b_1]\times\dots\times[a_d,b_d]\subseteq\mathbb{R}^d$. Then there exists a $3$ -layer neural network $N(\bx)$ with depth $2$ and width $2d$ such that $N(\bx) = 1$ for $\bx\in[a_1+\gamma,b_1-\gamma]\times\dots\times[a_d+\gamma,b_d-\gamma]$, $N(\bx)=0$ for $\bx\notin A$ and $|N(\bx)|\leq 1$ for all $\bx\in\mathbb{R}^d$.
\end{lemma}

\begin{proof}
We define:
\[
N(\bx) = \sigma\left(1-\frac{1}{\gamma}\sum_{i=1}^d\sigma(a_i+\gamma-x_i) - \frac{1}{\gamma}\sum_{i=1}^d\sigma(x_i - b_i + \gamma)\right)~,
\]
it is a $3$-layer neural network with width $2d$. Note that for every $i\in[d]$, if $x_i<a_i$ or $x_i > b_i$ then $N(\bx)=0$. Also, if $\bx\in[a_1+\gamma,b_1-\gamma]\times\dots\times[a_d+\gamma,b_d-\gamma]$ then $N(\bx) =1$, and finally for values of $\bx$ not specified above, $N(\bx)$ interpolates between $0$ and $1$, hence $|N(\bx)| \leq 1$ for these values.
\end{proof}

\begin{lemma}\label{lem:approx lipschitz with NN}
Let $h:[0,1]^d\rightarrow\mathbb{R}$ an $L$-Lipschitz function with $\sup_{\bx\in[0,1]^d}|h(\bx)|\leq C$, and $n\in\mathbb{N}$. Then there is a $3$-layer neural network $N(\bx)$ with width $n^d\cdot 2d$ such that $\int_{{[0,1]^d}}|N(\bx)-h(\bx)|d\bx\leq \frac{2C+L\sqrt{d}}{n^d}$.
\end{lemma}

\begin{proof}
First, we split the hypercube $[0,1]^d$ into $n^d$ smaller and equally sized hypercubes, in the following form: for every $i_1,\dots,i_d\in\{0,\dots,n-1\}$ we define the hypercube $\Big[\frac{i_1}{n}, \frac{i_1+1}{n}\Big]\times\dots\times \left[\frac{i_d}{n}, \frac{i_d+1}{n}\right]$, there are $n^d$ such hypercubes, each with volume $n^{-d}$. Denote these hypercubes as $A_1,\dots,A_M$ for $M=n^d$. For each $A_i$, pick any $\bx_i\in A_i$ and let $c_i = h(\bx_i)$. For each $A_i=\Big[\frac{i_1}{n}, \frac{i_1+1}{n}\Big]\times\dots\times \left[\frac{i_d}{n}, \frac{i_d+1}{n}\right]$ we use \lemref{lem:approx cube with NN} with $\gamma=n^{-2d}$ to get a neural network $N_i(\bx)$ such that $N_i(\bx)=0$ for $\bx\notin A_i$, $N_i(\bx)$=1 for $\bx\in \Big[\frac{i_1}{n}+n^{-2d}, \frac{i_1+1}{n} - n^{-2d}\Big]\times\dots\times \left[\frac{i_d}{n} + n^{-2d}, \frac{i_d+1}{n} - n^{-2d}\right]$, and $|N_i(\bx)|\leq 1$ for every $\bx\in[0,1]^d$. Then we have that:
\begin{align*}
    \int_{A_i}|h(\bx)-c_i\cdot N_i(\bx)|d\bx \leq 2Cn^{-2d} + L\sqrt{d}n^{-2d}~,
\end{align*}
where we used that inside $A_i$, if $N_i(\bx)=c_i$ then $|h(\bx)-c_i\cdot N_i(\bx)|\leq \frac{L\sqrt{d}}{n^d}$ since $h(\bx)$ is $L$-Lipschitz, otherwise $|h(\bx)-c_i\cdot N_i(\bx)|\leq 2C$ but the area for which this happens is at most $n^{-2d}$. Define $N(\bx):=\sum_{i=1}^Mc_iN_i(\bx)$, then $N(\bx)$ is a $3$-layer neural network with width $n^d\cdot 2d$, and we get that:
\begin{align*}
    \int_{{[0,1]^d}}|N(\bx)-h(\bx)|d\bx \leq \sum_{i=1}^M\int{{A_i}}|N_i(\bx)-h(\bx)|d\bx \leq n^d\cdot\left( 2Cn^{-2d} + L\sqrt{d}n^{-2d}\right) \leq \frac{2C+L\sqrt{d}}{n^d}~.
\end{align*}
\end{proof}

We are now ready to prove the main theorem:


\begin{proof}[Proof of \thmref{thm:weak learnability vary depth}]
Fix $n\in\mathbb{N}$, in the proof we denote $\theta_0:=\theta_0^n$ and $\theta_t$ to be $\theta_0$ after $t$ iterations of gradient descent, and denote $\nabla_T = \nabla_{\theta_T}\left(L_{f_n}(g^n_{\theta_T})\right)$. We have that:
\begin{align}\label{eq:theta T - theta 0 second time}
    \norm{\theta_{T+1} - \theta_0}^2 &= \norm{\theta_T - \eta\nabla_T - \theta_0} = \norm{\theta_T - \theta_0}^2 + \eta^2 \norm{\nabla_T}^2 - 2\eta\inner{\nabla_T,\theta_T - \theta_0} \nonumber\\
    &\leq \norm{\theta_T - \theta_0}^2 + \eta^2\norm{\nabla_T}^2 + 2\eta\norm{\nabla_T}\cdot \norm{\theta_T-\theta_0} \nonumber\\
    & \leq \dots \leq \sum_{t=0}^T\eta^2\norm{\nabla_t}^2 + 2\eta \norm{\nabla_t}\cdot \norm{\theta_t-\theta_0}
\end{align}
We will first bound the norm of the gradient at each iteration. Suppose that $\norm{\theta_t-\theta_0}\leq \frac{1}{n}$, then, by the assumption on the initialization, each coordinate of $\nabla_{\theta_t}\left(g^n_{\theta_t}(\bx)\right)$ is an $L$-Lipschitz function from $[0,1]^d$ to $\mathbb{R}$ with $\sup_{\bx\in[0,1]^d} \left(\nabla_{\theta_t}\left(g^n_{\theta_t}(\bx)\right)\right)_i \leq L$ for $i\in[p(n)]$. By \lemref{lem:approx lipschitz with NN} there is a $3$-layer neural network $N_i:[0,1]^d\rightarrow\mathbb{R}$ with width $\alpha(n)^d 2d$ such that

\[
\int_{{[0,1]^d}}|N_i(\bx)-\left(\nabla_{\theta_t}g^n_{\theta_t}(\bx)\right)_i|d\bx\leq \frac{4L+2L \sqrt{d}}{\alpha(n)^d}\leq \frac{6L\sqrt{d}}{\alpha(n)^d}~.
\]
For every $i\in[p(n)]$ (each coordinate of $\nabla_{\theta_t}\left(g^n_{\theta_t}(\bx)\right)$) we have that:
\begin{align}\label{eq:bound on correlation f_n and gradient of g theta}
    & \left|\mathbb{E}_{\bx\sim U\left([0,1]^d\right)}[\left(\nabla_{\theta_t}g^n_{\theta_t}(\bx)\right)_if_n(\bx)]\right| \nonumber\\
    \leq &  \left|\mathbb{E}_{\bx\sim U\left([0,1]^d\right)}[N_i(\bx)f_n(x)]\right| + \mathbb{E}_{\bx\sim U\left([0,1]^d\right)}\left[\left|\left(\nabla_{\theta_t}g^n_{\theta_t}(\bx)\right)_i-N_i(\bx)\right|\cdot |f_n(\bx)|\right] \nonumber\\
    \leq & \frac{1}{\alpha(n)} + \frac{6LC\sqrt{d}}{\alpha(n)^d} \leq \frac{7LC\sqrt{d}}{\alpha(n)}~.
\end{align}
Here we used the assumption that ReLU neural networks with at most $\alpha(n)^d 2d$ cannot weakly approximate $f_n$ in the sense stated in the theorem, which for the case of hinge loss, means that the correlation between the two functions is bounded by $\alpha(n)^{-1}$.

Using \eqref{eq:bound on correlation f_n and gradient of g theta} we can now bound the norm of the gradient at each iteration, under the assumption that $\norm{\theta_t-\theta_0}\leq \frac{1}{n}$:
\begin{align}\label{eq:bound on gradient nabla t}
    \norm{\nabla_t}^2 & = \left\|\frac{\partial}{\partial \theta}\mathbb{E}_{\bx\sim U([0,1]^d)}\left[\max\{0,1-g_{\theta_t}^n(\bx)f_n(\bx) \} \right]\right\|^2 \nonumber\\
    & \leq \sum_{i=1}^{p(n)} \left|\mathbb{E}_{x\sim U([0,1])}\left(\nabla_{\theta_t}g_{\theta_t}^n(\bx)\right)_i f_n(\bx) \right|^2 \leq \frac{7LC\sqrt{d}\cdot p(n)}{\alpha(n)}~.
\end{align}

Denote $a:=\frac{7LC\sqrt{d}\cdot p(n)}{\alpha(n)}$ and assume w.l.o.g that $a<1$ (otherwise, take a larger $n$), combining \eqref{eq:theta T - theta 0 second time} and \eqref{eq:bound on gradient nabla t} we get:
\begin{align*}
    \norm{\theta_{T+1} - \theta_0}^2 \leq T\eta^2 a + 2\eta \sqrt{a}\sum_{t=0}^T\norm{\theta_t - \theta_0} \leq \max\{1,\eta^2\}\left(Ta + 2\sqrt{a}\sum_{t=0}^T\norm{\theta_t - \theta_0}\right)~.
\end{align*}
We will show using induction that $\norm{\theta_{T+1}- \theta_0}^2 \leq \max\{1,\eta^2\}3aT^2$. For $T=0$ it is clear. Assume for all $t\leq T$, then we have:
\begin{align}\label{eq:bound on distance theta T theta 0}
    \frac{1}{\max\{1,\eta^2\}}\cdot\norm{\theta_{T+1} - \theta_0}^2 &\leq Ta + 2\sqrt{a}\sum_{t=0}^T\sqrt{3at^2}\leq Ta + 2\sqrt{3}a\sum_{t=0}^T t \nonumber\\
    & \leq Ta + \frac{2\sqrt{3}aT^2}{2}\leq T^2a + \sqrt{3}aT^2 \leq 3aT^2~.
\end{align}
Hence, if $\max\{1,\eta^2\}3aT^2\leq \frac{1}{n}$ then the above applies since the function $g_{\theta_T}^n(x)$ is Lipschitz and bounded, this applies for all $T \leq \sqrt{\frac{\alpha(n)}{21LC\max\{1,\eta^2\}\sqrt{d}p(n)}}$. Let $T$ be bounded as above, then:
\begin{align}
    L_{f_n}(g_{\theta_0}^n) &= \mathbb{E}_{x\sim U([0,1]^d)}\left[\max\{0,1-g_{\theta_0}^n(\bx)f_n(\bx)\}\right] \nonumber\\
    & = \mathbb{E}_{\bx\sim U([0,1]^d)}\left[\max\{0,1-g_{\theta_{T+1}}^n(\bx)f_n(\bx) + (g_{\theta_{T+1}}^n(\bx) - g_{\theta_0}^n(\bx))f_n(\bx)\}\right]\nonumber \\
    & \leq \mathbb{E}_{\bx\sim U([0,1]^d)}\left[\max\{0,1-g_{\theta_{T+1}}^n(\bx)f_n(\bx)\right] + \mathbb{E}_{\bx\sim U([0,1]^d)}\left[\left|(g_{\theta_0}^n(\bx) - g_{\theta_{T+1}}^n(\bx))f_n(\bx)\right|\right] \nonumber \\
    & \leq  L_{f_n}(g_{\theta_{T+1}}^n) + L\norm{\theta_{T+1} - \theta_0}\cdot\mathbb{E}_{\bx\sim U([0,1]^d)}[f_n(\bx)] \nonumber \\
    & \leq L_{f_n}(g_{\theta_{T+1}}^n) + \frac{21L^2C^2\max\{1,\eta^2\}\sqrt{d}p(n)}{\alpha(n)}T^2~,
\end{align}
where we used \eqref{eq:bound on distance theta T theta 0} and that $|f_n(\bx)| \leq C$ for all $\bx\in[0,1]^d$. This proves that:
\[
L_{f_n}(g_{\theta_0}^n) - L_{f_n}(g_{\theta_T}^n) \leq  \frac{21L^2C^2\max\{1,\eta^2\}\sqrt{d}p(n)}{\alpha(n)}T^2~,
\]
and in particular, if $T$ is polynomial in $n$, then $\cf$ is not weakly learnable with gradient descent.

\end{proof}

\section{Proof of \thmref{thm:linear_classes}}\label{appen:proofs from weak learning}

To prove the theorem, we first need the following technical lemma (due to \cite{szorenyi2009characterizing}):
\begin{lemma}
\label{lem:sq_correlation}
Let $f_1, \dots, f_d$ such that $\abs{\inner{f_i,f_j}} < \frac{1}{d}$ for every $i \ne j$. Fix some $h : \cx \to [-1,1]$, and some $\tau > 0$. Then, the number of functions from $f_1,\dots, f_d$ for which $\abs{\inner{f_j,h}} \ge \tau$ is at most $2(\tau^2-\frac{1}{d})^{-1}$, i.e.:
\[
\abs{\{j \in [d] ~:~ \abs{\inner{f_j,h}} \ge \tau\}} \le 2(\tau^2-1/d)^{-1}
\]
\end{lemma}
\begin{proof}
We define $A = \{j \in [d] ~:~ \inner{f_j,h} \ge \tau\}$, and note that:
\[
\tau \abs{A} \le \sum_{j \in A} \inner{h, f_j} = \inner{h, \sum_{j \in A} f_j}
\]
We also have:
\begin{align*}
\inner{h, \sum_{j \in A} f_j}^2_\cd
&\le \norm{h}^2_\cd \norm{\sum_{j \in A} f_j}^2_\cd 
\le \inner{\sum_{j \in A} f_j, \sum_{j \in A} f_j} 
= \sum_{j \in A} \left(1 + \sum_{j' \in A, j' \ne j} \inner{f_j,f_{j'}} \right) \\
&\le \abs{A} \left(1 + \frac{\abs{A}}{d} \right) = \abs{A} + \frac{\abs{A}^2}{d}
\end{align*}
So, we have $\abs{A} \le (\tau^2-1/d)^{-1}$. Similarly, for $A' = \{j \in [d] ~:~ \inner{f_j,\Psi_i} \le -\tau\}$ we get $\abs{A'} \le (\tau^2-1/d)^{-1}$. Therefore, the required follows.
\end{proof}

We use the following Theorem, which is an extension of the result in \cite{blum1994weakly}:
\begin{theorem}\label{thm:sq-learnability}
Let $\cf$ be a class of functions over $\cx$ and let $\cd$ be a distribution such that $\SQdim(\cf,\cd) \ge d \ge 16$. Let $\ell$ be the hinge-loss. Then any statistical-query algorithm with tolerance at least $1/d^{1/3}$, needs at least $\frac{1}{8}d^{1/3}$ queries to learn $\cf$ with loss less that $1-\frac{2}{\sqrt{d}}$.
\end{theorem}
\begin{proof}
Let $f_1, \dots, f_d$ be the maximal set of functions with $\abs{\inner{f_i,f_j}_\cd} \le \frac{1}{d^3}$. Now, let $g : \{\pm 1\}^n \times \{\pm 1\} \to [-1,1]$
be some statistical-query, and denote $C_g = \E_{\x \sim \cd, y \sim \{\pm 1\}} \left[g(\x,y)\right]$.
We say that some function $f_i$ is consistent with $C_g$ if:
\[
\abs{\E_{\x \sim \cd} \left[ g(\x,f_i(\x)) \right] - C_g} \le \frac{1}{d^{1/3}}
\]
Denote $g_+, g_-$ such that $g_+(\x) = g(\x,1)$ and $g_-(\x) = g(\x,-1)$. 
Now, observe that:
\begin{align*}
\E_{\x \sim \cd} \left[g(\x, f_i(\x)) \right]- C_g
&= \E_{\x \sim \cd} \left[ \ind_{f_i(\x) = 1} g_+(\x) + \ind_{f_i(\x) = -1} g_+(\x) - \frac{1}{2} g_+(\x) - \frac{1}{2} g_-(\x)\right] \\
&= \E_{\x \sim \cd} \left[ \frac{1}{2}(1+f_i(\x)) g_+(\x) + \frac{1}{2}(1-f_i(\x)) g_-(\x) - \frac{1}{2} g_+(\x) - \frac{1}{2} g_-(\x)\right] \\
&= \frac{1}{2}\E_{\x \sim \cd} \left[ f_i(\x) g_+(\x) - f_i(\x) g_-(\x) \right] \\
&= \frac{1}{2} \inner{f_i, g_+} - \frac{1}{2} \inner{f_i, g_-}
= \inner{f_i, \frac{1}{2}(g_+-g_-)}
\end{align*}

Denote $\bar{g} = \frac{1}{2}(g_+-g_-)$ and since $\bar{g}(\x) \in [-1,1]$ we get from \ref{lem:sq_correlation} that the number of functions from $f_1, \dots, f_d$ that are not consistent with $C_g$ is at most $2 (1/d^{2/3}-1/d)^{-1} \le 4 d^{2/3}$. Now, let $\ca$ be some statistical-query algorithm, and let $g_1, \dots, g_{k}$ be $k$ queries made by $\ca$, upon receiving responses $C_{g_1}, \dots, C_{g_{k-1}}$, and let $h$ be the hypothesis returned by $\ca$ after these queries. Now, the number of functions $f_i$ that are consistent with all the responses $C_{g_1}, \dots, C_{g_{k-1}}$ is at most $4k d^{2/3}$. So, if $\ca$ makes at most $\frac{1}{8} d^{1/3}$ queries, then at least $1/2$ of the functions $f_1, \dots, f_d$ are consistent with the responses, and assume w.l.o.g. that $f_1, \dots, f_{d/2}$ are consistent with the responses. Now, let $\tilde{h}$ be the clipping of $h$ to $[-1,1]$, so:
\[
\tilde{h}(\x) = \begin{cases}
-1 & h(\x) < -1 \\
h(\x) & h(\x) \in [-1,1] \\
1 & h(\x) > 1
\end{cases}
\]
Notice that for the hinge-loss $\ell$ we have for every $\x$ and $y \in \{\pm 1\}$ that $\ell(h(\x), y) = \max \{1-h(\x)y, 0\} \ge 1-y\tilde{h}(\x)$.
Indeed, we have the following cases:
\begin{itemize}
\item If $h(\x) \in [-1,1]$ then $\max\{1-h(\x)y,0\} = 1-h(\x)y = 1-\tilde{h}(\x)y$.
\item If $h(\x) y > 1$ then $\max \{1-h(\x)y, 0\} = 0 = 1- y \tilde{h}(\x)$.
\item If $h(\x) y < -1$ then $\max \{1-h(\x)y,0\} = 1-h(\x)y \ge 1- \tilde{h}(\x)y$.
\end{itemize}
Therefore, the hinge-loss of $h$ with respect to some function $f_i$ is:
\begin{align*}
L_{f_i(\cd)}(h)
&= \E_{\x \sim \cd} \left[\ell(h(\x), f_i(\x))\right] \\
&\ge \E_{\x \sim \cd} \left[1-f_i(\x) \tilde{h}(\x)\right] \\
&= 1- \inner{f_i,\tilde{h}}
\end{align*}
From Lemma \ref{lem:sq_correlation}, there are at most $(4/d-1/d)^{-1} =d/3<d/2$ functions in $f_1, \dots, f_d$ with $\inner{f_i,\tilde{h}} \ge \frac{2}{\sqrt{d}}$. So, there exists a function $f_i$ that is consistent with all the responses of the oracle and has $\inner{f_i,\tilde{h}} < \frac{2}{\sqrt{d}}$. For this function we get:
\begin{align*}
L_{f_i(\cd)}(h)
&\ge 1-\inner{f_i,\tilde{h}} \ge 1-\frac{2}{\sqrt{d}}
\end{align*}
\end{proof}

\begin{proof} of Theorem \ref{thm:linear_classes}. 
Assume $\cf$ can be weakly approximated with respect to $\cd$ by a polynomial-size kernel class $\ch$. Then, from Theorem \ref{thm:sq_weak_approx}, there exists a polynomial $p(n)$ such that $\SQdim(\cf_n,\cd_n) \le p(n)$. Then, the algorithm that returns $f_i$ with $i = \arg \max_{j \in [d]} \abs{\inner{f_j, f^*}}$ is an efficient statistical-query algorithm that weakly learns $\cf$ with respect to $\cd$ and the hinge-loss (since the returned function is a binary function, the hinge-loss and the zero-one loss are equivalent).

In the other direction, assume that $\cf$ is efficiently weakly learnable from statistical-queries. So, there exists some polynomial $p$ and a statistical-query algorithm that returns a loss with error $\le 1-\frac{1}{\abs{p(n)}}$ for every $f \in \cf_n$, using a polynomial number of queries and polynomial tolerance. Therefore, from Theorem \ref{thm:sq-learnability}, there exists a (positive) polynomial $q$ such that $\SQdim(F_n,\cd_n) \le q(n)$. So, let $f_1, \dots, f_{q(n)}$ be the maximal set such that $\abs{\inner{f_i,f_j}} \le \frac{1}{q(n)^3}$ for every $i \ne j$. Denote $\Psi_n(\x) = (f_1(\x), \dots, f_{q(n)}(\x))$, and so we have $f_i \in \ch_{\Psi_n}^1$ for every $i$. So, for every $f \in F$ we have:
\begin{align*}
\min_{h \in \ch_{\Psi_n}^1} L_{f(\cd)}(h)
&\le \min_{i \in [q(n)]} L_{f(\cd)}(f_i) 
= \min_{i \in [q(n)]} \E_{\cd} \ell(f_i(\x), f(\x)) \\
&= \min_{i \in [q(n)]} \E_{\cd} \left[1- f_i(\x)f(\x)\right]
= L_\cd(\Zero) - \max_{i \in [q(n)]} \abs{\inner{f_i,f}} \le L_\cd(\Zero) - \frac{1}{q(n)^3}
\end{align*}
\end{proof}

\section{Proofs from \secref{sec:strong separation}}\label{appen:proofs from strong separation}

\subsection{Telgarski's Function}\label{appen: proof of deep  separation}
Recall we consider approximating the following function for $n\in\mathbb{N}$:
\[
f_n(\x) = \begin{cases} 1 & \exists t \in \naturals, ~x_1 \in [\frac{2t}{2^{n}}, \frac{2t+1}{2^n}] \\ -1 & otherwise \end{cases} 
\]
In words, we split the interval $[0,1]$ to $2^n$ intervals each of length $\frac{1}{2^n}$, on even intervals $f_n(x)=1$ and on odd intervals $f_n(x)=-1$. For the rest of the proof we consider the distribution family $\mathcal{F} = \left\{\mathcal{F}_n\right\}_{n\in\mathbb{N}}$ defined by the functions $f_n$.
We have that:

\begin{lemma}\label{lem:telgarsky realizable}
Let $\mathcal{H} = \{\mathcal{H}_n\}$ be the hypothesis classes defined by the sign function of ReLU neural networks with depth $2n$ and width $2$, and $\mathcal{F}$ as above. Then $\mathcal{F}$ is realizable by $\mathcal{H}$.
\end{lemma}

\begin{proof}
Following Lemma 3.10 from \cite{telgarsky2016benefits}, we define $h':\reals\rightarrow\reals$ by $h'(x) = \sigma(2\sigma(x) - 4\sigma(x-1/2) - 1/2)$ where $\sigma$ is the ReLU function. Note that $h'$ interpolates between $(0,-1/2)$, $(1/2,1/2)$ and $(1,-1/2)$. We define $h:\reals^d\rightarrow\reals$ by $h(\bx) = \sign(h'\circ\cdots \circ h'(x_1))$ where the composition is done $n$ times. Note that $h(\bx)$ is a $2n$-layer neural network and that $h(\bx) = f_n(\bx)$ (the first and last interval might be cut in half, in this case we can shift $x_1\mapsto x_1 - 1/2^{n+1}$).
\end{proof}

In order to show that the functions above cannot be weakly approximated with a shallow neural network, we show that a 1-dimensional ReLU network is a piecewise linear function, and bound the number of pieces.

\begin{lemma}\label{lem:ReLU NN is piecewise linear}
Let $N(x):[0,1]\rightarrow\mathbb{R}$ be a ReLU network with depth $L$ and width $k$. Then $N(x)$ is a piecewise linear function with at most $2^{L-1}k^L$ pieces.
\end{lemma}

\begin{proof}
First, let $g_1,g_2:[0,1]\rightarrow\mathbb{R}$ be two piecewise linear functions with corresponding $k_1,k_2$ pieces. Denote $h(x)= g_1(x) + g_2(x)$, we show that $h(x)$ is also a piecewise linear function with at most $k_1+k_2$ pieces. Let $A_1$ (resp. $A_2$) be a partition of $[0,1]$ such that on each interval in $A_1$ (resp. $A_2$) the function $g_1$ (resp. $g_2$) is linear. Take $B$ to be a partition which is a refinement of $A_1$ and $A_2$ in the following way: Denote $A_1 = \{[a_0,a_1],\dots,[a_{k-1},a_k]\}$ and $A_2 = \{[a'_0,a'_1],\dots,[a'_{k-1},a'_k]\}$. We construct $B = \{[b_0,b_1],\dots,[b_{m-1},b_m]\}$ such that $b_0 = a_0 = 0$ (note that also $a'_0=0$), $b_1 = \min\{a_1,a'_1\}$ and inductively for every $i >1$ take $b_i = \min\{a_j,a'_{j'}\}$ where $j$ and $j'$ are the smallest index for which there is no $l,l'<i$ with $b_l = a_j$ and $b_{l'} = a'_{j'}$. Note that $B$ has at most $k_1+k_2$ intervals, since each boundary point of an interval in $B$ contains exactly one unique boundary point from an interval in either $A_1$ or $A_2$. Also, in each interval of $B$ the function $h(x)$ is linear, because it is linear in both $g_1$ and $g_2$. Hence $h(x)$ is piecewise linear with at most $k_1+k_2$ pieces.

Next, we can write $N(x)$ in the following way: Let $a_0=x$, $a_1 = \sigma(U_1a_0 + b_1)$ where $U_1\in\mathbb{R}^{k\times 1},~b_1\in\mathbb{R}^k$, and for $1<i<L$ $a_i = \sigma(U_i a_{i-1} + b_i)$ where $U_i\in\mathbb{R}^{k\times k},~ b_i\in\mathbb{R}^k$, finally $a_L = U_L a_{L-1} + b_L$ for $U_L\in \mathbb{R}^{1\times L},~b_L\in\mathbb{R}$. We show that each coordinate of $a_i$ for $1\leq i \leq L$ is a piecewise linear function of $x$ using induction, and bound the amount of pieces by $k^L$. For $i=0$ it is clear, assume it is true for $i-1$. Because $a_i = \sigma(U_i a_{i-1} + b_i)$, the $j-th$ coordinate of $a_i$ is equal to $ \sigma\left(\sum_{l=1}^k (U_i)_{jl}(a_{i-1})_l + (b_i)_l\right)$, by the induction hypothesis for every $l$, $(a_{i-1})_l$ is a piecewise linear function of $x$ with at most $k^{i-1}$ pieces. By the previous claim, $\sum_{l=1}^k (U_i)_{jl}(a_{i-1})_l + (b_i)_l$ is a sum of $k$ piecewise linear functions, each with $(2k)^{i-1}$ pieces hence it has at most $2^{i-1}k^i$ pieces. 

Finally note that since the ReLU function is a piecewise linear function with two pieces, composing it with a linear function is a piecewise linear function with at most two pieces. Hence, composing ReLU with a piecewise linear function with $k$ pieces can turn every linear segment into at most two linear segments, meaning that the composition is a piecewise linear function with at most $2k$ pieces. This means that $ \sigma\left(\sum_{l=1}^k (U_i)_{jl}(a_{i-1})_l + (b_i)_l\right)$ is a piecewise linear function with $(2k)^i$ pieces. 

Using the above for $i=L$ we have that $N(x)$ is a piecewise linear function with at most $2^{L-1}k^L$ pieces (note that the last layer does not contain ReLU activation).
\end{proof}

We can now prove that the functions above cannot be weakly approximated by neural networks with depth less than $\sqrt{n}$:

\begin{theorem}
\label{thm:deep_depth_separation}
Let $k: \naturals \to \naturals$ some function such that $k(n) \le \sqrt{n}$. Then the sign function of every polynomial-size depth-$k(n)$ dimension-$d$ network \textbf{cannot} weakly approximate $\cf := \{\cf_n\}_{n \in \naturals}$.
\end{theorem}

\begin{proof}
Let $g:[0,1]\rightarrow\{-1,1\}$ be a function with at most $k$ jumps (i.e. change of output from $-1$ to $1$ or from $1$ to $-1$), and let $n\in\mathbb{N}$ with $2^{n-1}>k$. We split the interval $[0,1]$ into intervals $\left[\frac{2i}{2^n},\frac{2(i+1)}{2^n}\right]$ for $i=0,\dots, 2^{n-1}-1$, there are $2^{n-1}$ such intervals. Since $g(x)$ changes sign at most $k$ times there are at least $2^{n-1}-k$ intervals on which $g(x)$ is constant. Hence, we have that:
\begin{align}\label{eq:loss between telgarsky and few jumps}
    L_{\mathcal{F}_n}(g) = \mathbb{E}_{x\sim U([0,1])}\left[\max\{0,1-f_n(x)\cdot g(x)\}\right] \geq 1\cdot \frac{(2^{n-1}-k)}{2^{n-1}} + 0\cdot \frac{k}{2^{n-1}} = \frac{(2^{n-1}-k)}{2^{n-1}}
\end{align}

For any $\y\in\reals^{d-1}$ we define $p_\y:\reals\rightarrow\reals^d$ by $p_\y(x) = (\y,x)$. Let $N(\bx)$ be a  depth-$L$ dimension-$d$ network with width $p(n)$ for some polynomial $p$, and note that if we fix $\y\in\reals^{d-1}$ then $N\circ p_\y:\reals\rightarrow\reals$ is a depth-$L$ dimension-$1$ network with width at most $2p(n)$. Hence, by \lemref{lem:ReLU NN is piecewise linear} $N\circ p_\y$ is a piecewise linear function with at most $2^{L-1}(2k)^L$ pieces, which shows that the number of jumps of its sign has the same bound. Since we assume the dimension is constant and the width is polynomial in $n$, we can assume that for large enough $n$, $p(n)>d$. Now, using \eqref{eq:loss between telgarsky and few jumps} we get:
\begin{align*}
    L_{f_n(\cd_n)}(N) &=  \mathbb{E}_{\bx\sim U([0,1]^d)}\left[\max\{0,1-f_n(\bx)\cdot N(\bx)\}\right] \\
    &= \int_{\y\sim U([0,1]^{d-1})}\left(\int_{x\sim U([0,1])}\max\{0,1-f_n\circ p_y(x)\cdot N\circ p_y(x)\}dx\right)d\y\\
    & \geq \frac{(2^{n-1}-2^{\sqrt{n}-1}(2p(n))^{\sqrt{n}})}{2^{n-1}} = 1 - \frac{2^{\sqrt{n}}(2p(n))^{\sqrt{n}}}{2^n}~.
\end{align*}


In particular, for any polynomial $q(n)$ we have that $L_{f_n(\cd_n)}(N) \geq 1-1/q(n)$
\end{proof}

The proof of \thmref{thm:deep network seperation} now directly follows from \lemref{lem:telgarsky realizable} and \thmref{thm:deep_depth_separation}.

\subsection{Parity Functions}\label{appen:parity functions}

For the following result we use the notion of the statistical-query (SQ) dimension (first introduced in \cite{blum1994weakly}). This notions is a measure of complexity of a target class with respect to a given distribution, which counts the number of "almost-orthogonal" function in the target class with respect to the inner product induces by the distribution:

\begin{definition}
Let $\cf_n$ be some target class and let $\cd_n$ be some distribution over $\cx_n$. We define the statistical-query dimension of $(\cf_n,\cd_n)$, denoted $\SQdim(\cf_n,\cd_n)$, to be the largest number $d\in\naturals$ for which there exist functions $f_1, \dots, f_d \in \cf_n$ such that for every $i \ne j$ we have:
\[
\abs{\inner{f_i,f_j}_{\cd_n}} := \abs{\E_{\x \sim \cd_n} \left[f_i(\x)f_j(\x)\right]} < \frac{1}{d}
\]
\end{definition}

While the SQ-dimension was introduced as a measure of bounding the complexity of statistical-query algorithms, we show that this measure can also directly bound the approximation power of kernel classes. In particular, we show that if the SQ-dimension of some target class and distribution is super-polynomial, then they cannot be weakly approximated by a polynomial-size kernel class. In the following results we assume that $\mathcal{X}_n = \{\pm 1\}^n$.

\begin{theorem}
\label{thm:sq_weak_approx}
Let $\ch$ be some polynomial-size kernel class, let $\cf = \{\cf_n \}_{n \in \naturals}$ a sequence of target classes and let $\cd = \{\cd_n \}_{n \in \naturals}$ be a sequence of distributions over $\{\cx_n\}_{n \in \naturals}$. Denote $d(n) = \SQdim(\cf_n,\cd_n)$. Then, if $d(n)$ is super-polynomial, $\ch$ cannot weakly approximate $\cf$ with respect to $\cd$.
\end{theorem}

The theorem directly connects between the ability of kernel classes to approximate a target class, and its SQ-dimension. The proof relies on the fact that for a fixed mapping $\Psi_n:\mathcal{X}_n\rightarrow\mathcal{Y}_n$ and a set of "almost-orthogonal" functions $f_1,\dots,f_{d(n)}$, it cannot happen that the correlation between between $\Psi_n$ and all of the $f_i$'s is large. Hence, given many such $f_i$'s, it is possible to find one which $\Psi_n$ cannot approximate well, in other words, the more $f_i$'s there are, it is harder for $\Psi_n$ to approximate them all at once.

The proof of Theorem \ref{thm:sq_weak_approx} is largely based on the following key lemma:
\begin{lemma}\label{lem:linear_hardness}
Fix some $\Psi : \mathcal{X} \to [-1,1]^N$, and define:
\[
\mathcal{H}_\Psi^B = \{\x \to \inner{\Psi(\x), \bw} ~:~ \norm{\bw}_2 \le B\}
\]
Then, if $d(n) > 3$, there exist $f_1, \dots, f_{d(n)} \in \cf_n$ such that:
\[
\E_{j \sim [d(n)]} \left[\min_{h \in \mathcal{H}_\Psi^B} L_{f_j(\cd_n)}(h)\right] \ge 1-\frac{\sqrt{5N}B}{d(n)^{1/12}}
\]
\end{lemma}
\begin{proof}
Let $f_1, \dots, f_{d(n)} \in \cf_n$ be the set of functions realizing $\SQdim(\cf_n, \cd_n)$, and we denote $d := d(n)$. For some $j \in [d]$, 
let $\mathcal{L}_j(\bw) := L_{f_j(\cd_n)}(\inner{\Psi(\x),\bw})$ and define the objective $G_j(\bw) := \mathcal{L}_j(\bw) + 
\frac{\lambda}{2}\norm{\bw}^2$. Observe that for every $i \in [N]$ we have:
\[
\frac{\partial}{\partial w_i} G_j(0) = \E_{\x \sim \mathcal{D}_n} \left[ f_j(\x) \Psi_i(\x) \right]
\]
And so:
\begin{equation} \label{eqn:grad}
\begin{split}
\E_{j \sim [d]} \left[ \norm{\nabla G_j(0)}^2 \right]
&= \E_{j \sim [d]} \left[ \sum_{i \in [N]} \left(\frac{\partial}{\partial w_i} G_j(0) \right)^2 \right] \\
&= \E_{j \sim [d]} \left[ \sum_{i \in [N]} \E_{\x \sim \mathcal{D}} \left[ f_j(\x) \Psi_i(\x) \right]^2 \right] \\
&= \sum_{i \in [N]} \frac{1}{d} \sum_{j=1}^{d}  \E_{\x \sim \mathcal{D}} \left[ f_j(\x) \Psi_i(\x) \right]^2
\end{split}
\end{equation}
We define $A = \{j \in [d] ~:~ \abs{\inner{f_j,\Psi_i}_\cd} \ge \tau\}$, and from Lemma \ref{lem:sq_correlation} we have $\abs{A} \le 2(\tau^2-\frac{1}{d})^{-1}$. Therefore, for $\tau = d^{-1/3}$ we get:
\begin{align*}
\sum_{j=1}^{d}  \E_{\x \sim \mathcal{D}} \left[ f_j(\x) \Psi_i(\x) \right]^2 
&= \sum_{j \in A} \inner{f_j,\Psi_i}_\cd^2 + \sum_{j \notin A} \inner{f_j,\Psi_i}^2_\cd  \\
&\le 2 (\tau^2 -1/d)^{-1} + d \tau
\le 2(d^{-2/3} - 1/d)^{-1} + d^{2/3} \le 5 d^{2/3}
\end{align*}
And plugging into (\ref{eqn:grad}) we get:
\[
\E_{j \sim [d]} \left[ \norm{\nabla G_j(0)}^2 \right]
= \sum_{i \in [N]} \frac{1}{d} \sum_{j=1}^{d}  \E_{\x \sim \mathcal{D}} \left[ f_j(\x) \Psi_i(\x) \right]^2 \le \frac{5 N}{d^{1/3}}
\]
Using Jensen inequality we get:
\begin{equation}\label{eqn:grad_bound}
\E_{j \sim [d]} \left[ \norm{\nabla G_j(0)} \right]^2 \le \E_{j \sim [d]} \left[ \norm{\nabla G_j(0)}^2 \right] \le \frac{5N}{d^{1/3}}
\end{equation}
Note that $G_j$ is $\lambda$-strongly convex, and therefore, for every $\bw, \bu$ we have:
\[
\inner{\nabla G_j(\bw) - \nabla G_j(\bu), \bw - \bu} \ge \lambda \norm{\bw - \bu}^2
\]
Let $\bw^*_j := \arg \min_\bw G_j(\bw)$, and so $\nabla G_j (\bw^*_j) = 0$. Using the above we get:
\[
\lambda \norm{\bw^*_j}^2 \le \inner{\nabla G_j(\bw^*_j) - \nabla G_j(0), \bw_j*} \le \norm{\nabla G_j(0)} \norm{\bw_j^*}
\Rightarrow \norm{\bw_j^*} \le \frac{1}{\lambda}\norm{\nabla G_j(0)}
\]
Now, notice that $\mathcal{L}_j$ is $\sqrt{N}$-Lipschitz, since: 
\[
\norm{\nabla\mathcal{L}_j(\bw)} = \norm{\nabla \mean{\ell(y,\inner{\Psi(\x),\bw})}} \le \mean{\abs{\ell'}\norm{\Psi(\x)}} \le \sqrt{N}
\]
Therefore, we get that:
\begin{equation}\label{eqn:lower_bound_w_star}
1 - \mathcal{L}_j(\bw_j^*) = \mathcal{L}_j(0)- \mathcal{L}_j(\bw_j^*) \le \sqrt{N} \norm{\bw_j^*} \le \frac{\sqrt{N}}{\lambda} \norm{\nabla G_j(0)}
\end{equation}
Denote $\hat{\bw}_j = \arg \min_{\norm{\bw} \le B} \mathcal{L}_j(\bw)$, and by optimality of $\bw^*_j$ we have:
\begin{equation}\label{eqn:upper_bound_w_star}
\mathcal{L}_I(\bw^*_j) \le \mathcal{L}_j(\bw^*_j) + \frac{\lambda}{2}\norm{\bw_j^*}^2 \le \mathcal{L}_j(\hat{\bw}_j) + \frac{\lambda}{2} \norm{\hat{\bw}_j}^2 \le\mathcal{L}_j(\hat{\bw}_j) + \frac{\lambda B^2}{2}
\end{equation}
From \ref{eqn:lower_bound_w_star} and \ref{eqn:upper_bound_w_star} we get:
\begin{equation}
1- \frac{\sqrt{N}}{\lambda} \norm{\nabla G_j(0)} \le \mathcal{L}_j(\bw_j^*) \le \mathcal{L}_j(\hat{\bw}_j) + \frac{\lambda B^2}{2}
\end{equation}
Taking an expectation and plugging in \ref{eqn:grad_bound} we get:
\[
\E_{j \sim [d]} \left[\min_{h \in \mathcal{H}_\Psi^B} L_{f_j(\mathcal{D})}(h)\right]
= \E_{j \sim [d]} \left[\mathcal{L}_j(\hat{\bw}_j)\right] \ge 
1-\frac{\sqrt{N}}{\lambda}\E_{j\sim [d]} \left[\norm{\nabla G_j(0)}\right] - \frac{\lambda B^2}{2} \ge 1-\frac{\sqrt{5}N}{\lambda d^{1/6}} - \frac{\lambda B^2}{2}
\]
Since this is true for all $\lambda$, taking $\lambda = \frac{\sqrt{2\sqrt{5}N}}{d^{1/12}B}$ we get:
\[
\E_{j \sim [d]} \left[\min_{h \in \mathcal{H}_\Psi^B} L_{f_j(\mathcal{D})}(h)\right] \ge 1-\frac{\sqrt{2\sqrt{5}N}B}{d^{1/12}}
\]
\end{proof}

Given the above Lemma, the proof of the Theorem \ref{thm:sq_weak_approx} is immediate:
\begin{proof} of Thm. \ref{thm:sq_weak_approx}.

Fix some polynomial-size kernel class $\ch := \{\ch_n\}_{n \in \naturals}$, with mappings $\Psi_n : \cx_n \to [-1,1]^{p(n)}$ and $\ch_n = \ch_{\Psi_n}^{q(n)}$ for some polynomials $p,q$. Then, from Lemma \ref{lem:linear_hardness}, for every $n$ we have:
\[
\max_{f \in \cf_n} \min_{h \in \ch_n} L_{f(\cd_n)}(h) \ge \E_{j \sim [d]} \min_{h \in \ch_n} L_{f_j(\cd_n)}(h)  = \E_{j \sim [d]} \left[\min_{h \in \mathcal{H}_\Psi^{q(n)}} L_{f_j(\cd)}(h)\right] \ge 1-\frac{\sqrt{5p(n)}q(n)}{d(n)^{1/12}}
\]
Since for every $\cd$ we have $L_\cd(\Zero) = 1$ (we use the hinge-loss), we get that:
\[
\max_{f \in \cf_n} \min_{h \in \ch_n} L_{f(\cd_n)}(\Zero) - L_{f(\cd_n)}(h) \le \frac{\sqrt{5p(n)}q(n)}{d(n)^{1/5}}
\]
And so, for every polynomial $r$ we have:
\[
\sup_{n \in \naturals} \max_{f \in \cf_n} \min_{h \in \ch_n} \abs{r(n)}\left(L_{f(\cd)}(\Zero) - L_{f(\cd)}(h)\right) \le \sup_{n \in \naturals} \frac{\sqrt{5p(n)}q(n)\abs{r(n)}}{d(n)^{1/12}} = 0
\]
Which proves the theorem.
\end{proof}

Now, we can use Theorem \ref{thm:sq_weak_approx} to show a strong separation between \textbf{any} polynomial-size kernel class and the class of polynomial-size shallow neural-network. All we need is to find a target class of functions which have super polynomial SQ-dimension, but can be realized by $2$-layer neural networks. 

We will use \textbf{parity functions} over $n$-bits: For some subset $I \subseteq [n]$, denote by $f_I(\x) = \prod_{i \in I} x_i$, the parity over the bits of $I$, let $\cf_n = \{f_I ~:~ I \subseteq [n]\}$ and $\cd_n$ the uniform distribution on $\cx_n$.

\begin{proof}[Proof of \thmref{thm:parity strong separation}]
By Lemma $5$ from \cite{shalev2017failures} parity functions are realizable by $2$-layer neural networks with width and weight magnitude linear in the input dimension.
On the other hand, it is easy to see that every two parity functions on a different subset on $\cx_n$ are orthogonal, and there are $2^n$ such functions. Hence, by \thmref{thm:sq_weak_approx} parity functions are not weakly approximated by a polynomial-size kernel class.
\end{proof}

\section{Proofs from \secref{sec:depth 2-3}}

\subsection{Proof of \thmref{thm:learnability depth2-3}}\label{append:learnability depth-2-3}
We first show use a concentration inequality to find a large set of vectors with pairwise small inner product.
\begin{lemma}
There exists a set of vectors $\z^{(1)}, \dots, \z^{(d)} \in \mathcal{X}_n$ of size $d = 2^{n/12}$ such that for every $i \ne j$ we have $\sum_{t=1}^n \ind \{z_t^{(i)} \ne z_t^{(j)}\} \ge \frac{n}{4}$.
\end{lemma}

\begin{proof}
Fix some $d$, draw $d$ vectors uniformly from $\mathcal{X}$, and denote them $Z := (\z^{(1)}, \dots, \z^{(m)})$. Fix some $i \ne j$, and denote $S_{i,j} = \sum_{t=1}^n \ind\{z^{(i)}_t \ne z^{(j)}_t\}$. Notice that:
$\E_Z \left[ S_{i,j}  \right] = \sum_t \prob{z_t^{(i)} \ne z_t^{(j)}} = \frac{n}{2}$, and by Hoeffding's inequality:
\[
\prob{S_{i,j} \le \frac{n}{4}} \le \exp\left(- \frac{n}{8} \right)
\]
There are $\frac{d^2}{2}$ choices for $i \ne j$, and so using the union bound we get:
\[
\prob{\forall i \ne j:~ S_{i,j} > \frac{n}{4}} \ge 1-\frac{d^2}{2} e^{-n/8}
\]
Since $2^{4/3} \le e$, choosing $d = 2^{n/12} \le e^{n/16}$ we get that with probability at least $\frac{1}{2}$ over the choice of $Z$, for every $i \ne j$ we have $S_{i,j} \ge \frac{n}{4}$. Therefore, the required follows.
\end{proof}

\begin{proof} of Theorem \ref{thm:learnability depth2-3}.
Fix $\z^{(1)}, \dots, \z^{(d)} \in \mathcal{X}_n$ with $d = 2^{n/12}$ such that $\sum_{t=1}^n \ind \{z_t^{(i)} \ne z_t^{(j)}\} \ge \frac{n}{4}$ for all $i \ne j$. Fix some $i \ne j$ and observe that:
\begin{align*}
\abs{\inner{F_{\z^{(i)}}, F_{\z^{(j)}}}} &= \abs{\E_{\x,\z} \left[ \prod_{t \in I(\z^{(i)})} (x_t \vee z_t) \prod_{t \in I(\z^{(j)})} (x_t \vee z_t) \right]} \\
&= \abs{ \prod_{t \in I(\z^{(i)}) \triangle I(\z^{(j)})} \E(x_t \vee z_t)} = \prod_{t \in I(\z^{(i)}) \triangle I(\z^{(j)})} \frac{1}{2} \le 2^{-n/4} = \frac{1}{d^3}
\end{align*}
Where $I \triangle J = I \setminus J \cup J \setminus I$, and using the fact that $\abs{I(\z^{(i)}) \triangle I(\z^{(j)})} = \abs{\{t\in [n] ~:~ z_t^{(i)} \ne z_t^{(j)}\}} \ge \frac{n}{4}$. This shows that $\SQdim(\cf, \cd) \ge 2^{n/12}$, combining with \thmref{thm:sq-learnability} finishes the proof.
\end{proof}

\subsection{Proof of \thmref{thm:parity-like separates}}\label{appen:proofs of 2-3-layer}

We will prove the theorem in a more generalized setting. We start by showing that any target class with large SQ-dimension can be used to construct a target distribution that is hard to approximate using a 2-layer network. This construction is very natural: given a class of functions $\cf_n$ from $\cx_n$ to $\{\pm 1\}$, we consider a function from $\cx_n \times \cz_n$ to $\{\pm 1\}$, for some space $\cz_n$, where we identify every function in $f' \in \cf_n$ with an element $\varphi_n(f') \in \cz_n$. Then, the new function $f(\x,\z)$ is just applying $\varphi_n^{-1}(\z)$ on the input $\x$. We call this the induced function:

\begin{definition}
Let $\cx_n, \cz_n$ be two input spaces. 
Let $\cf_n$ be some target class of functions from $\cx_n$ to $\cy$, with $\abs{\cf_n} = \abs{\cz_n}$ and let $\varphi_n : \cz_n \to \cf_n$ be some bijection. The tuple $(\cf_n, \varphi_n)$ naturally induces a function $f : \cx_n \times \cz_n \to \cy$, where $f(\x,\z) = \varphi_n(\z)(\x)$. We denote by $F(\cf_n,\varphi_n) := f$ the induced function.
\end{definition}

A simple example is where $\cx_n = \cz_n = \{\pm 1\}^n$, we can think of all the functions of the form $f_{\bz}:\cx_n\rightarrow\{\pm 1\}$ defined by $f_{\bz}(\bx) = \prod_{i=1}^n(x_i \vee z_i)$. It is natural to identify each function $f_{\bz}$ with a vector $\bz\in\cz_n$, hence the induced function $f:\cx_n\times\cz_n\rightarrow\{\pm 1\}$ is defined similarly by $f(\bx,\bz) = \prod_{i=1}^n(x_i \vee z_i)$, where $\bz$ is an input vector instead of a constant. 

We define a depth-two neural network $g : \mathcal{X} \times \mathcal{Z} \to \reals$ to be any function of the form: 
\[
g(\x, \z) = \sum_{i=1}^k u_i \sigma\left(\inner{\bw^{(i)}, \x} + \inner{\bv^{(i)},\z} + b_i\right)
\]
where $\sigma$ is some $1$-Lipschitz activation. Note that this is the exact same definition of depth-2 neural network we had before, only here we explicitly split the weights to accommodate the two input vectors $\bx$ and $\bz$.

We show that the induced function can be used to generate a distribution that is hard to weakly approximate using any depth-two network $g$:
\begin{theorem}
\label{thm:shallow_hardness}
Let $\cf = \{\cf_n \}_{n \in \naturals}$ be a sequence of target classes and let $\cd = \{\cd_n \}_{n \in \naturals}$ be a sequence of distributions over $\{\cx_n\}_{n \in \naturals}$, and denote $d(n) = \SQdim(\cf_n,\cd_n)$. For every $n$, let $\cz_n\subseteq \{\pm 1\}^n$ be some input space with with $|\cz_n| = |\cf_n|$, and fix a bijection $\varphi_n: \cz_n\rightarrow\cf_n$.
Then, if $d(n)$ is super-polynomial, there exists a sequence of distributions $\cd'$ over $\cx_n \times \cz_n$ such that the sequence of induced target functions $\{F(\cf_n,\varphi_n)\}_{n \in \naturals}$ cannot be weakly approximated by any polynomial-size depth-two networks class with respect to $\cd'$.
\end{theorem}

\begin{proof}[Proof of \thmref{thm:parity-like separates}]
The proof that $\{\cf_n\}$ cannot be weakly approximated by 2-layer neural networks follows immediately from \thmref{thm:shallow_hardness}. We just need to take $\cz_n = \cx_n = \{\pm 1\}^n$, and $\varphi_n$ to be the natural bijection $\bz\mapsto F_n(\bz,\bx)$, and using the fact that parity over $n$-bits w.r.t the uniform distribution has an exponential SQ-dimension.

The proof that $\{\cf_n\}$ can be realized by depth-3 neural networks is known, e.g. see \cite{malach2019learning}.
\end{proof}

We move on to proving \thmref{thm:shallow_hardness}. Fix some target class $\cf_n$ and some distribution $\cd_n$ such that $d(n) = \SQdim(\cf_n,\cd_n)$. Fix some $\varphi_n : \cz_n \to \cf_n$ and let $f := F(\cf_n,\varphi_n)$. Let $f_1, \dots, f_{d(n)} \in \cf_n$ be functions that realize the $\SQdim$. We define a distribution $\cd'_n$ over $\cx_n \times \cz_n$ such that 
\[
\cd_n'(\x,\z) = \begin{cases} \frac{1}{d(n)} \cd_n(\x) & \varphi_n(\z) \in \{f_1, \dots, f_{d(n)}\} \\
 0 & otherwise \end{cases}~.
\]
That is, we sample $\x \sim \cd_n$ and sample uniformly $\z \sim \{\varphi_n^{-1}(f_1), \dots, \varphi_n^{-1}(f_{d(n)})\}$.
 
 As a first step for proving the Theorem, we show that such distribution cannot be approximated by a polynomial-size depth-two neural network, where the weights $\bv$ take integer values:

\begin{lemma}
\label{lem:shallow_approx_integer}
Assume that there exists $\Delta > 0$ such that $v_j^{(i)} \in \Delta \integers := \{\Delta \cdot z ~:~ z \in \integers \}$ for every $i,j$ and $\norm{\bu^{(i)}},\norm{\bw^{(i)}},\norm{\bv^{(i)}}, \norm{\bb} < R$.
\[
L_{f(\cd'_n)}(g) = \mean{\ell(g(\x,\z),f(\x,\z))} \ge 1-\frac{3\sqrt{10k}R^{5/2}n^{3/4}}{\sqrt{\Delta} d(n)^{1/12}}
\]
\end{lemma}
\begin{proof}
For every $\z \in \mathcal{X}$ denote $j(\z) = \inner{\bv^{(i)},\z}$, and since $v_j^{(i)} \in [-R,R]\cap \Delta \integers$, we get that $j(\z) \in [-R\sqrt{n},R\sqrt{n}] \cap \Delta\integers$. Indeed, fix some $i$ and we have $\frac{1}{\Delta} \bv^{(i)} \in \integers^n$, and since $\z \in \integers^n$ we have $\frac{1}{\Delta} j(\z) = \inner{\frac{1}{\Delta} \bv^{(i)}, \z} \in \integers$.
Define $\Psi_{i,j}(\x) = \frac{1}{3R\sqrt{n}}\sigma\left(\inner{\bw^{(i)},\x} + j + b_i\right)$ for every $i \in [k]$ and $j \in [-R\sqrt{n},R\sqrt{n}]\cap \Delta\integers$, and note that:
\[
\abs{\Psi_{i,j}(\x)} \le \frac{1}{3R\sqrt{n}}\abs{\inner{\bw^{(i)},\x} + j + b_i}
\le \frac{1}{3R\sqrt{n}} \left(\norm{\bw^{(i)}}\norm{\x}) + \abs{j} + \abs{b_i} \right) \le 1
\]
Notice that $\abs{\left[-R\sqrt{n}, R\sqrt{n} \right]\cap \Delta \integers} \le 2\left\lfloor \frac{R\sqrt{n}}{\Delta} \right \rfloor$, and so there are at most $2 \left\lfloor \frac{R\sqrt{n}}{\Delta} \right \rfloor$ choices for $j$.
Denote $N := 2k\lfloor \frac{R\sqrt{n}}{\Delta} \rfloor$ and let $\Psi : \mathcal{X} \to [-1,1]^N$ defined as $\Psi(\x) = [\Psi_{i,j}(\x)]_{i,j}$ (in vector form). Denote $B = 3R^2\sqrt{n}$, and from Lemma \ref{lem:linear_hardness} we have:
\begin{align*}
\E_{\cd'_n} \left[\min_{\norm{\hat{\bu}}\le B} \ell(\inner{\hat{\bu},\Psi(\x)}, f(\x,\z))\right] 
&= \E_{j \sim [d(n)]} \left[\min_{\norm{\hat{\bu}}\le B} \ell(\inner{\hat{\bu},\Psi(\x)}, f_j(\x))\right] \\
&= \E_{j \sim [d(n)]} \left[\min_{h \in \mathcal{H}_\Psi^B} L_{\mathcal{D}, f_j}(h)\right] \ge 1-\frac{\sqrt{5N}B}{d(n)^{1/12}}
\end{align*}
Notice that $g(\x,\z) = \sum_{i=1}^k 3R\sqrt{n}u_i \Psi_{i,j(\z)}(\x) = \inner{\bu(\z), \Psi(\x)}$ where $\bu(\z)_{i,j} = \begin{cases} 3R\sqrt{n}u_i & j = j(\z) \\ 0 & j \ne j(\z) \end{cases}$. Since $\norm{\bu(\z)} \le 3R \sqrt{n} \norm{\bu} \le B$ we get that:
\[
\mean{\ell(g(\x,\z), f(\x,\z))} = \mean{\ell(\inner{\bu(\z), \Psi(\x)}, f(\x,\z))}
\ge \E \left[\min_{\norm{\hat{\bu}}\le B} \ell(\inner{\hat{\bu},\Psi(\x)}, f(\x,\z))\right] \ge 1-\frac{\sqrt{5N}B}{d(n)^{1/12}}
\]
\end{proof}

Now, we can extend this result to general polynomial-size depth-two networks. This is done by correctly rounding the weights to get integer values, and use the previous lemma.

\begin{lemma}
\label{lem:shallow_hardness}
Assume $\norm{\bu^{(i)}},\norm{\bw^{(i)}},\norm{\bv^{(i)}}, \norm{\bb} \le R$. Then:
\[
L_{f(\cd)}(g) \ge 1 - \frac{6\sqrt{k}R^{2}n^{5/6}}{d(n)^{1/18}}
\]
\end{lemma}

\begin{proof}
Fix some $\Delta \in (0,1)$, and let $\hat{\bv}^{(i)} = \Delta \left \lfloor \frac{1}{\Delta} \bv^{(i)} \right\rfloor \in \Delta \integers^n$, where $\left\lfloor \cdot\right\rfloor$ is taken element-wise.
Notice that for every $j$ we have:
\[
\abs{v^{(i)}_j - \hat{v}^{(i)}_j} = \abs{v^{(i)}_j - \Delta \left\lfloor \frac{1}{\Delta} v^{(i)}_j \right \rfloor} = \Delta \abs{\frac{1}{\Delta} v^{(i)}_j - \left\lfloor \frac{1}{\Delta} v^{(i)}_j \right \rfloor} \le \Delta
\]
Observe the following neural network:
\[
\hat{g}(\x, \z) = \sum_{i=1}^k u_i \sigma\left(\inner{\bw^{(i)}, \x} + \inner{\hat{\bv}^{(i)},\z} + b_i\right)
\]
For every $\x, \z \in \mathcal{X}$, using Cauchy-Schwartz inequality, and the fact that $\sigma$ is $1$-Lipchitz:
\begin{align*}
\abs{g(\x,\z) - \hat{g}(\x,\z)} &\le \norm{\bu} \sqrt{ \sum_{i=1}^k \abs{\sigma\left(\inner{\bw^{(i)}, \x} + \inner{\bv^{(i)},\z} + b_i\right) - \sigma\left(\inner{\bw^{(i)}, \x} + \inner{\hat{\bv}^{(i)},\z} + b_i\right)}^2} \\
&\le \norm{\bu} \sqrt{ \sum_{i=1}^k \abs{\inner{\bv^{(i)},\z} -  \inner{\hat{\bv}^{(i)},\z}}^2} \\
&\le \norm{\bu} \sqrt{ \sum_{i=1}^k \norm{\bv^{(i)}- \hat{\bv}^{(i)}}^2\norm{\z}^2} \le R \sqrt{k} \Delta n
\end{align*}
Now, by Lemma \ref{lem:shallow_approx_integer} we have:
\[
L_{f(\cd'_n)}(\hat{g}) \ge 1-\frac{3\sqrt{10k}R^{5/2}n^{3/4}}{\sqrt{\Delta} d(n)^{1/12}}
\]
And using the fact that $\ell$ is $1$-Lipschitz we get:
\begin{align*}
L_{f(\cd'_n)}(g) &= \mean{\ell(g(\x,\z),f(\x,\z))} \\
&\ge \mean{\ell(\hat{g}(\x,\z),f(\x,\z))} - \mean{\abs{\ell(g(\x,\z),f(\x,\z))-\ell(\hat{g}(\x,\z),f(\x,\z))}} \\
&\ge L_{f(\cd'_n)}(\hat{g}) - \mean{\abs{g(\x,\z)-\hat{g}(\x,\z)}} \\
&\ge 1-\frac{3\sqrt{10k}R^{5/2}n^{3/4}}{\sqrt{\Delta} d(n)^{1/12}} - R \sqrt{k} \Delta n
\end{align*}
This is true for any $\Delta > 0$, so we choose $\Delta = \frac{\sqrt[3]{45/2}R}{d(n)^{1/18}n^{1/6}}$ and we get:
\[
L_{f(\cd)}(g) \ge 1 - \frac{\sqrt[3]{180}\sqrt{k}R^{2}n^{5/6}}{d(n)^{1/18}}
\]
\end{proof}

\begin{proof} of Theorem \ref{thm:shallow_hardness}.
Immediate from Lemma \ref{lem:shallow_hardness}.
\end{proof}

\end{document}